\documentclass{ecai} 

\usepackage[cmex10]{amsmath}

\usepackage[algo2e,linesnumbered]{algorithm2e}
\usepackage{verbatim}
\usepackage{algorithm}

\usepackage{newfloat}
\usepackage{listings}

\usepackage{url}
\usepackage{amsfonts}
\usepackage{soul}
\usepackage{color}
\usepackage{dsfont}
\usepackage{xcolor}
\usepackage{mathtools}
\usepackage[normalem]{ulem}

\lstnewenvironment{myverbatim}[1][]{%
  \lstset{
    basicstyle=\ttfamily,
    #1
  }%
}{}

\usepackage{fancybox}

\makeatletter
\newenvironment{CenteredBox}{%
\begin{Sbox}}{
\end{Sbox}\centerline{\parbox{\wd\@Sbox}{\TheSbox}}}
\makeatother

\usepackage{accents}
\newlength{\dhatheight}
\newcommand{\doublehat}[1]{%
    \settoheight{\dhatheight}{\ensuremath{\hat{#1}}}%
    \addtolength{\dhatheight}{-0.35ex}%
    \hat{\vphantom{\rule{1pt}{\dhatheight}}%
    \smash{\hat{#1}}}}

\usepackage{titlesec}

\newcommand{\CRC}[1]{#1}

\DeclarePairedDelimiter{\paren}{(}{)}           
\DeclarePairedDelimiter{\ang}{\langle}{\rangle} 
\DeclarePairedDelimiter{\norm}{\lVert}{\rVert}  
\DeclarePairedDelimiter\bkt{[}{]}             
\DeclarePairedDelimiter{\set}{\{}{\}}           
\DeclareMathOperator*{\argmax}{arg\,max}
\DeclareMathOperator*{\argmin}{arg\,min}

\SetKwFunction{AcceptableActs}{AcceptableActs}
\SetKwFunction{ActionSelection}{ActionSelection}
\SetKwFunction{FindDirection}{FindDirection}
\DeclareMathOperator{\accacts}{\AcceptableActs}
\DeclareMathOperator{\actselect}{\ActionSelection}
\DeclareMathOperator{\finddir}{\FindDirection}
\DeclareMathOperator{\Up}{\textit{Up}}
\DeclareMathOperator{\Down}{\textit{Down}}
\DeclareMathOperator{\Right}{\textit{Right}}
\DeclareMathOperator{\Left}{\textit{Left}}

\newcommand{\RVector}{\mathbf{R}}
\newcommand{\Real}{\mathbb{R}}
\newcommand{\Nat}{\mathbb{N}}

\newcommand{\projc}[2]{\mathrm{proj}^{\Delta}(#1, #2)} 

\SetKw{Or}{or}
\SetKw{And}{and}
\SetKw{Not}{not}
\SetKw{None}{None}

\newcommand{\qstar}{Q^{\star}}
\newcommand{\qshat}{\hat{Q}^{\star}}
\newcommand{\expect}{\mathop{\mathbb{E}}}

\newcommand{\defeq}{\overset{\Delta}{=}}

\newcommand{\realnumbers}{\mathbb{R}}

\newcommand{\cref}[1]{Chapter~\ref{#1}}  

\newcommand{\traject}{\zeta}
\newcommand{\maze}{\mathrm{MAZE}}
\newcommand{\vv}[1]{\boldsymbol{#1}}
\newcommand{\projp}[2]{\boldsymbol{P}_{\boldsymbol{#2}}^{H}(\boldsymbol{#1})} 
\newcommand{\projs}[2]{\boldsymbol{P}_{\boldsymbol{#2}}^{S^+}(\boldsymbol{#1})} 

\newcommand{\vx}{\boldsymbol{x}}
\newcommand{\vu}{\boldsymbol{u}}

\newcommand{\projcs}[2]{\mathrm{proj}^{\Delta}(#1, #2)} 

\newtheorem{theorem}{Theorem}[section]
\newtheorem{proof}{Proof}[section]
\newtheorem{proposition}{Proposition}[section]
\newtheorem{remark}{Remark}[section]

\newtheorem{lemma}[theorem]{Lemma}

\usepackage{amsmath,amssymb,amsfonts}
\usepackage{algorithmic} 
\usepackage{graphicx}
\usepackage{textcomp}
\usepackage{xcolor}
\usepackage{paralist}
\usepackage{wrapfig}

\usepackage{latexsym}

\newcommand{\BibTeX}{B\kern-.05em{\sc i\kern-.025em b}\kern-.08em\TeX}

\begin{document}


\begin{frontmatter}


\paperid{2418} 


\title{Thresholded Lexicographic Ordered Multiobjective Reinforcement Learning}


\author[A]{\fnms{Alperen}~\snm{Tercan}\thanks{Corresponding Author. Email: tercan@umich.edu.}}
\author[B]{\fnms{Vinayak S.}~\snm{Prabhu}\thanks{Email:vinayak.prabhu@colostate.edu}}

\address[A]{University of Michigan}
\address[B]{Colorado State University}


\begin{abstract}
\looseness-1Lexicographic multi-objective problems, which impose  a  lexicographic
importance order over the objectives, arise in many real-life scenarios.
Existing Reinforcement Learning work  directly addressing lexicographic tasks has been scarce.
The few proposed approaches were all noted to be heuristics without theoretical guarantees as the Bellman equation is not applicable to them.
Additionally, the practical applicability of these prior approaches
also  suffers from various issues such as not being able to reach the goal state.
While some of these issues have been known before, in this work we
investigate further shortcomings and propose fixes for improving practical performance
in many cases.
We also present a  policy optimization approach using our Lexicographic Projection Algorithm (LPA) that has the potential to address these theoretical and practical concerns.
Finally, we demonstrate our proposed algorithms on benchmark problems.
\end{abstract}

\end{frontmatter}


\section{Introduction}

The need for multi-objective reinforcement learning (MORL) arises in many real-life scenarios and the setting cannot be reduced to single-objective reinforcement learning tasks in general \citep{vamplew2022scalar}. However, solving multiple objectives requires overcoming certain inherent difficulties.
In order to compare candidate solutions, we need to incorporate
given user preferences with respect to the different objectives.
This can lead to \textit{Pareto optimal} or non-inferior solutions forming a set of solutions where no solution is better than another in terms of all objectives.
%
%
Various methods of specifying  user preferences have been proposed and evaluated along
three main fronts: (a)~ expressive power, (b)~ease of writing, and (c)~the availability of methods for solving problems with such preferences.
Typically there are tradeoffs when a particular front is emphasized. For example, writing preference specifications that result in a partial order of solutions instead of a total order makes the specification easier for the user but may not be enough to describe a unique preference.
Three main motivating scenarios differing on when the user preference becomes available or used have been studied in the literature. (1) User preference is known beforehand and is incorporated into the problem \emph{a priori}. (2) User preference is used \emph{a posteriori}, i.e., firstly a set of representative Pareto optimal solutions is generated, and the user preference is specified over it. (3) An interactive setting where the user preference is specified gradually during the search and the search is guided accordingly.

Our present work is in the a priori setting.
  The most common specification method in this setting is  \emph{linear scalarization} which which requires the designer to assign weights to the objectives
and take a weighted sum, thus
making solutions comparable \citep{feinberg1994markov}. The main benefit of this technique is that it allows the use of many standard off-the-shelf algorithms as it preserves the additivity of the reward functions. However, expressing user preference with this technique requires significant domain knowledge and preliminary work in most scenarios \citep{li2019urban}.
While it can be the preferred method when the objectives can be expressed in comparable quantities, e.g. when all objectives have a monetary value, this is not the case most of the time. Usually, the objectives are expressed in incomparable quantities like money, time, and carbon emissions.
Additionally, a composite utility combining the various objectives, and an
approximation of that with linear scalarization limits us to a subset of the Pareto optimal set.

\looseness-1 To address these drawbacks of linear scalarization, several other approaches have been proposed and studied. Nonlinear scalarization methods like Chebyshev \citep{perny2010finding} are more expressive and can capture all of the solutions in the Pareto optimal set, however, they do not provide ease of writing. In this paper, we will focus on an alternative specification method that overcomes both limitations of linear scalarization, named \emph{Thresholded Lexicographic Ordering} (TLO) \citep{gabor1998multi} \citep{li2019urban}. In lexicographic ordering, the user determines an importance order for the objectives. The more important objectives dominate the lower order objectives. Given two candidate solutions, a lower ranked objective is used to rank the two solutions
only if all higher order objectives have the same values. The thresholding part of the technique allows a more generalized definition for being the same w.r.t. an objective. The user provides a threshold for each objective except the last, and the objective values are clipped at the corresponding thresholds. This allows the user to specify values beyond which they are indifferent to the optimization of an objective. There is no threshold for the last objective as it is considered an unconstrained open-ended objective.

\looseness-1Despite the strengths of the TLO specification method, the need for a specialized algorithm to use it in reinforcement learning (RL) has prevented it from being a common technique. The \emph{Thresholded Lexicographic Q-Learning} (TLQ) algorithm was proposed as such an algorithm and has been studied and used in several papers
\citep{li2019urban} \citep{hayes2020dynamic}. While it has been noted that this algorithm does not enjoy the convergence guarantees of its origin algorithm (Q-Learning), we found that its practical use is limited to an extent that has not been discussed in the literature before.
In this work, we investigate such issues of TLQ further. We also present a \emph{Policy Gradient algorithm} as a general solution that has the potential to address many of the shortcomings of TLQ algorithms.

\noindent\emph{\textbf{Our Contributions.}} Our main contributions in this work are as follows:
(1)~We demonstrate the shortcomings of existing TLQ variants on a new important problem class that was not known before.
To the best of our knowledge, only \citet{vamplew2011empirical} had previously shown a case where TLQ did not work.
We formulate a taxonomy of the problem space in order to give insights into TLQ's performance in different settings.
Using our taxonomy, we demonstrate a new significant class of problems on which TLQ fails.
We exhibit this new problematic class on a common control scenario where the primary objective is reaching a goal state and the
other secondary objectives evaluate trajectories taken to the goal. 
(2)~We propose a \emph{lexicographic projection algorithm} which computes a lexicographically optimal direction that optimizes the current unsatisfied highest importance objective
while preserving the values of more important objectives using projections onto
hypercones of their gradients. Our algorithm allows adjusting how conservative the new direction is w.r.t. preserved objectives and can be combined with first-order optimization algorithms like Gradient Descent or Adam. We also validate this algorithm on a simple optimization problem from the literature.
(3)~We explain how this algorithm can be applied to policy-gradient algorithms to solve Lexicographic Markov Decision Processes (LMDPs) and experimentally demonstrate the performance of a REINFORCE adaptation on the cases that were problematic for TLQ.

\section{Related Work}
%

\looseness-1\citep{gabor1998multi} was one of the first papers to investigate RL in multi-objective tasks with preference ordering.
It introduces TLQ as an RL algorithm to solve such problems.
%
%
\citet{vamplew2011empirical} showed that TLQ significantly outperforms Linear Scalarization (LS) when the Pareto front is globally concave or when most of the solutions lie on the concave parts. However, LS performs better when the rewards are not restricted to terminal states, because TLQ cannot account for the already received rewards. Later, \citet{roijers2013survey} generalized this analysis by comparing more approaches using a unifying framework. To our knowledge, \citep{vamplew2011empirical} is the only previous work that explicitly discussed shortcomings of TLQ. However, we found that TLQ has other significant issues that occur even outside of the problematic cases they analyze.

\citet{wray2015multi} introduced
slack-based Lexicographic MDP (LMDP)
and the Lexicographic Value Iteration (LVI) algorithm. These LMDPs
specifies slacks from optimal policy instead of absolute thresholds. In other words, the user specifies how much worse than the optimal an objective can be instead of simply how good it should be. For example, in slack setting, one would say: “I need to get there at most 20\% slower than the fastest I could”. In the thresholds setting, one would say: “I need to get there in 30 minutes”. These two settings require special methods. While \citet{wray2015multi} proved the convergence to desired
policy if slacks are chosen appropriately, such slacks are generally too tight to allow
user preferences. This is also observed in \citet{pineda2015revisiting} which
claimed that while ignoring these slack bounds
negates
the theoretical guarantees, the resulting algorithm still can be useful in practice.

\citet{li2019urban} investigated the use of Deep TLQ for urban driving. It showed that the TLQ version proposed in \citep{gabor1998multi} introduces additional bias which is especially problematic in function approximation settings like deep learning. Also, it depends on learning the true Q function, which can not be guaranteed. To overcome these drawbacks, it used slacks instead of the static thresholds
and proposed a different update function. While \citet{lex_urbandriving} also similarly utilize slacks, it defines them in terms of action probabilities instead of Q-values. 
%
%
\citet{hayes2020dynamic} used TLQ in a multi-objective multi-agent setting and proposed a dynamic thresholding heuristic to deal with the explosion of the number of thresholds to be set.

However, we discovered that these works on using a Q-learning variant with thresholded ordering perform very poorly in most cases due to non-Markovianity of the value function they try to learn. It is possible to bypass this issue by using policy gradient approaches as they do not require learning an optimal value function. In order to handle conflicting gradients, some modifications to the gradient descent algorithm are needed.
Recent work on modified gradient descent algorithms came mostly from Multi Task Learning literature, which could be considered a multiobjective optimization problem \citep{desideri:inria-00389811} \citep{sener2018multi} \citep{lin2019pareto} \citep{pmlr-v119-mahapatra20a} \citep{NEURIPS2021_9d27fdf2}\cite{hu2024pa2d}.
While these papers use similar ideas with our work, their setting is different than ours as they do not have any explicit importance order; hence, not applicable to our setting. \citet{UCHIBE20081447} have the most similar setting to ours in gradient-based algorithms. It considers a set of constraints with an unconstrained objective. Then, the gradient of the unconstrained objective is projected onto the positive half-space of the active (violated) constraints and adds a correction step to improve the active constraints. When no valid projection is found, the most
violated constraints are ignored until a valid projection exists. This is one of the main differences with our setting: As we have an explicit importance-order of the objectives, it is not acceptable to ignore a constraint without considering the importance order. Also, we project the gradients onto hypercones instead of hyperplanes, which is a hypercone with $\pi/2$ vertex angle. Thus, our algorithm allows varying degrees of conservative projections to prevent a decline in the constraints.


While there are many other recent works on Constrained Markov Decision Process (CMDPs) like \citep{wachi2020safe,junges2016safety}, their approaches are not \CRC{directly} applicable as an importance order over constraints is not allowed. \CRC{However, an LMDP problem can be reduced to a series of CMDP problems using the following observation. When all constraints can be satisfied simultaneously, the importance order between them becomes irrelevant for the optimal policy. Hence, the LMDP can be solved as a CMDP. Then, for a general LMDP problem where all constraints cannot be necessarily satisfied simultaneously, the main task is to find the maximum $k$ such that most important $k$ objectives can be satisfied together. The most straightforward way to achieve this would be starting by checking the feasibility of the most important constraint only. Then, we can keep adding the next constraint in the importance order until the problem becomes infeasible in order to find the maximum $k$. Note that this algorithm is similar to "linear search" and would require $O(K)$ CMDPs to be solved for an LMDP with $K$ objectives. One can implement different strategies that imitate other search algorithms like binary search.}

\citet{ijcai2022p476} propose both value-based and policy-based approaches for LMDPs. Their value-based approach is based on slacks like \citep{li2019urban} and they require using very small slack values. This protects their approach from the issues with having relaxations by limiting their setting to strict lexicographic order, i.e. not allowing any suboptimality for the more important objectives. For policy-based methods, they use Lagrangian relaxation and their setting is again a strict lexicographic ordering, i.e. it does not allow treating values above a threshold equal. As we are primarily interested in solving tasks with thresholded objectives, this paper does not address our question.

\CRC{
Finally, using RL with lexicographic ordering recently began to attract attention from other communities as well. For example, \citet{lexicographic-automata} use formal methods to construct single objective MDPs when all of the objectives are $\omega$-regular.}

\section{Background}\label{sec:background}


\noindent\textbf{Multiobjective Markov Decision Process (MOMDP).}
A MOMDP is a tuple $\ang{S, A, P, \RVector, \gamma}$ where\hspace*{-4mm} 
  \begin{compactitem}
    \item $S$ is the finite set of states with initial state $s_{init} \in S$ and a set of terminal states $S_F$,
    \item $A$ is a finite set of actions,
    \item $P $: $S \times A \times S \to [0,1]$ is the  state transition function given by $P(s, a, s')
      =$ $ \mathbb{P}(s'|s,a)$, the probability of transitioning to state $s'$ given
 current state $s$ and action $a$.
    \item $\RVector = \bkt{R_1, \dots, R_K}^{T}$ is a vector that specifies the reward of transitioning from state $s$ to $s'$ upon taking action $a$ under $K$ different reward functions $R_i: S \times A \times S \to \Real$ for $i \in \set{1, \dots, K}$.
    \item $\gamma \in \Real$ is a discount factor.
    \end{compactitem}

In such a MOMDP, a finite \textit{trajectory} $\traject \in (S \times A)^*\times S$ is a sequence $\traject=s_0a_0s_1a_1\dots a_{T-1}s_T$ where $s_i \in S$, $a_i \in A$ and indices denote the time steps.
The evolution of an MDP is governed by repeated agent-environment interactions, where
  in each step, an \emph{agent} first picks actions in a state $s$ according to some probabilistic
  distribution, and for each of these actions $a$  the \emph{environment} generates
  next states according to $ \mathbb{P}(s'|s,a)$.
    Each reward function $R_i$ corresponds to an \emph{objective} $o_i$, the discounted
  rewards sum that the agent tries to maximize.
Control over a MOMDP requires finding an optimal \emph{policy} function $\pi^* : S \times A \to \bkt{0,1}$ which assigns
probability  $ \mathbb{P}_{\pi^*}(a|s)$ to actions $a\in A$.
In this paper, we use the \emph{episodic} case of MDP where the agent-environment interaction consists of sequences that start in $s_{init}$ and terminates when a state in $S_F$ is visited. The length of the episode, T, is finite but not a fixed number. In MDP literature, this is known as "indefinite-horizon" MDP.
The episodic case can be ensured by restricting ourselves to suitable policies which have a non-zero probability for all actions in all states. 

We define the quality of a policy $\pi$ with respect to an objective $o_i \in \set{1, \ldots, K}$ by the value function $V^{\pi}_i: S \to \Real$ given by
$  V^{\pi}_i(s) = \expect_{\pi}\bkt{\sum_{t=0}^{T} \gamma^t R_i(s_t, a_t, s_{t+1}) | s_0 = s}$.
Intuitively, $v^{\pi}_i(s)$ is the expected return from following policy $\pi$ starting from state $s$ w.r.t. objective $o_i$. The overall \emph{quality} of a policy $\pi$ is given by the vector valued function $\vv{V^{\pi}}: S \to \Real^K$ which is defined as $\vv{V^{\pi}}(s) = \bkt{V^{\pi}_1(s), \dots, V^{\pi}_K(s)}^{T}$.
As $\vv{V}$ is vector-valued, without a preference for comparing $V^{\pi}_i$ values across
different $i$,
we only have  a partial order over the range of  $\vv{V}$, leading to  \emph{Pareto front}
of equally good quality vectors. Further preference specification is needed to order the points on the Pareto front. A \emph{Lexicographic MDP (LMDP)} is a class of MOMDP which provides such an
ordering. It adds another component to MOMDP definition:
\begin{itemize}
  \item $\mathbf{\tau}=\ang{\tau_1, \ldots, \tau_{K-1}} \in \Real^{K-1}$ is a tuple of \emph{threshold values} where $\mathbf{\tau}_i$ indicates the threshold value for objective $o_i$ beyond which there is no benefit for $o_i$. The last objective does not require a threshold; hence, there are only $K-1$ values.
    Then, $\vv{\tau}$ can be used to compare value vectors $\vv{u}, \vv{v} \in \Real^K$ by defining the thresholded lexicographic comparison  $>^{\vv{\tau}}$ as
    $ \mathbf{u} >^{\vv{\tau}} \mathbf{v}$ iff there exists $ i \leq K$ such that:
    \begin{compactitem}

    \item $\forall j < i $ we have $\mathbf{u_j} \geq \min(\mathbf{v_j}, \tau_j)$; and
      \begin{compactitem}
    \item
      if $i<K$ then $\min(\mathbf{u_i}, \tau_i) > \min(\mathbf{v_i}, \tau_i)$,

    \item otherwise if $i=K$ then $\mathbf{u_i} > \mathbf{v_i}$.
    \end{compactitem}
  \end{compactitem}

\end{itemize}

Intuitively, we compare $\vv{u}$ and $\vv{v}$ starting from the most important objective
($j=1$); the less important objectives are considered only if the order of higher priority objectives
is respected.
%

Objectives $o_i$ for $i<K$ are said to be \emph{constrained objectives}. A constrained objective $o_i$  is said to be \emph{satisfied} when it is greater than or equal to its corresponding threshold $\tau_i$.
The relation $ \geq^{\vv{\tau}} $ is defined as $ >^{\vv{\tau}} \textrm{ OR }  =^{\vv{\tau}}$;  where $\vv{u} =^{\vv{\tau}} \vv{v}$ if $\min(\mathbf{u_i}, \tau_i) = \min(\mathbf{v_i}, \tau_i)$ for each constrained objective and $\vv{u}_K = \vv{v}_K$.

\noindent\textbf{Value-function Algorithms for Optimal Policies.}
An optimal policy $\pi^*$ is one that is better than or equal to any other policy, i.e., if $V^{\pi^*}(s) \geq^{\vv{\tau}} V^{\pi}(s) $  $\forall s \in S$ for all other
policies $\pi\in \Pi$ \citep{gabor1998multi}.
There are two approaches to finding optimal policies in RL: Value-function algorithms and Policy Gradient algorithms.
Value function based methods estimate the optimal action-value function $Q^*$ and construct $\pi^*$ using it. The action-value function under $\pi$, $Q^{\pi} : S \times A \to \Real^K$, is defined as:
%
  $Q^{\pi}(s,a) \defeq \expect_{\pi}\bkt{\sum_{t=0}^{T} \gamma^t \RVector(s_t, a_t, s_{t+1}) | s_0 = s, a_0=a}$.
%
The optimal action-value function, $\qstar$, is defined as:
%
 $ \qstar(s,a) = \max_{\pi \in \Pi} Q^{\pi} (s,a) $.
 Then, $\pi^*$ is obtained  as:
 $\pi^*(s,a) = 1$ if $a = \argmax_{a' \in A} \qstar(s,a') $, and $0$ otherwise.
%
In single objective MDPs, the Bellman Optimality Equation seen in Eq.~\ref{eq:bellman-optimality-expectation-version} is used to learn $\qstar$ as it gives an update rule that converges to $\qstar$ when applied iteratively.
\begin{equation}\label{eq:bellman-optimality-expectation-version}
\qstar(s, a) = \expect_{s'\sim P}\bkt{(R(s, a, s') + \gamma \max_{a' \in A} \qstar(s', a'))}
\end{equation}
Q-learning \citep{watkins1992q} is a very popular algorithm that takes this approach. TLQ tries to extend Q-learning for LMPDs; however, Bellman Optimality Equation does not hold in LMDPs. Hence, this approach lacks the theoretical guarantees enjoyed by Q-learning.



\noindent\textbf{Policy Gradient Algorithms for $\pi^*$.}
Policy gradient algorithms in RL try to learn the policy directly instead of inferring it from the value functions.
These methods estimate the gradient of the optimality measure w.r.t. policy and
update the candidate policy using this potentially imperfect information. We denote the policy parameterized by a vector of variables, $\theta$ as $\pi_{\theta}$.
The performance of the policy $\pi_{\theta}$, denoted  $J(\theta)$,  can be defined as the the expected return from following $\pi_{\theta}$ starting from $s_{init}$, i.e. $J(\theta) \defeq V^{\pi_{\theta}}(s_{init})$.
Once the gradient of the optimality measure w.r.t. the parameters of the policy function is estimated, we can use first-order optimization techniques like Gradient Ascent to maximize the optimality measure.
While all based on the similar theoretical results, a myriad of policy gradient algorithms have been proposed in the literature \citep{NIPS1999_464d828b} \citep{konda1999actor} \citep{schulman2017proximal} \citep{lillicrap2015continuous}.

\section{\!TLQ: Value Function Based Approaches for TLO\hspace*{-1.5mm}}\label{sec:tlq}
Previous efforts to find solutions to the LMDPs have been focused on value-function methods. Apart from \citep{wray2015multi}, which takes a dynamic programming approach, these have been variants of Thresholded Lexicographic Q-Learning (TLQ), an LMDP adaptation of Q-learning \citep{gabor1998multi} \citep{li2019urban}.
%
While these methods have been used and investigated in numerous papers, the extent of their weaknesses has not been discussed explicitly.






In order to adapt Q-learning to work with LMDPs, one cannot simply
use the update rule in Q-learning for each objective and learn the optimal value function of each objective completely independent of the others. Such an approach would
result in the actions being suboptimal for some of the objectives.
Based on this observation, two variants of TLQ \citep{gabor1998multi}\citep{li2019urban} have been proposed, which differ in how they take the other objectives into account. We analyze these variants by
dividing their approaches into two components: (1)~value functions and update rules;
and (2)~acceptable policies for action selection.
However, it should be noted that these components are inherently intertwined due to the nature of the problem --- the value functions and acceptable policies are defined recursively where each of them uses the other's last level of recursion. Due to these inherent circular referencing, the components will not be completely independent and some combinations of introduced techniques may not work.

\noindent\textbf{Value Functions and Update Rules.}
The starting point of learning the action-value function for both variants is $\qstar = \ang{\qstar_1, \ldots, \qstar_K}$ where each $\qstar_i: S \times A \to \Real$ is defined as in Section~\ref{sec:background} only with the change that the maximization is not done over the set of all policies $\Pi$ but over a subset of it $\Pi_{i-1} \subseteq \Pi$ which will be described below.
 \citet{gabor1998multi} propose learning $\qshat:S \times A \to \Real^K$ where each component of $\qshat$ denoted by $\qshat_i$ is defined as:
  $\qshat_i(s,a) \defeq \min(\tau_i, \qstar_i(s,a))$
In other words, it is the rectified version of $\qstar_i$. It is proposed to be learned by updating $\qshat_i(s,a)$ with the following value iteration which is adapted from Eq.~\ref{eq:bellman-optimality-expectation-version}
\begin{equation}\label{eq:update-gabor}
  \min\Big(
    \tau_i, \sum_{s'} P(s,a,s') (R_i(s,a,s') + \gamma \max_{\pi \in \Pi_{i-1}}
    \qshat_i\left( s',\pi\left(s'\right)  \right)
    \Big)
\end{equation}

\looseness-1 Notice that similar to the definition of $\qstar_i$, the main change during the adaptation is limiting the domain of $\max$ operator to $\Pi_{i-1}$ from $\Pi$.
On the other hand, \citet{li2019urban}
propose that we estimate $\qstar$ instead and use it when the actions are being picked.
This $\qstar$ uses the same update rule as Eq.~\ref{eq:bellman-optimality-expectation-version} with only change being maximization over $\Pi_{i-1}$.
%

\noindent\textbf{Acceptable Policies  $\Pi_i$ and Action Selection.}
The second important part of TLQ is the definition of "Acceptable Policies", $\Pi_i$, which is likely to be different for each objective.
The policies in $\Pi_i$ are ones that satisfy the requirements of the first $i$ objectives. Values of the acceptable policies in a given state are the acceptable actions in that state. Hence, these sets will be used as the domain of both $\max$ operator in the update rules and $\argmax$ operator in $\actselect$ function.
The pseudocode of this function can be seen in Algorithm~\ref{alg:action-selection}.
\begin{figure}
\begin{algorithm}[H]
{\small
  \SetAlgoLined
\SetKw{Break}{break}
\SetKwProg{Fn}{Function}{:}{}
\Fn{\ActionSelection{$s, Q, \epsilon | A$}}{
$r \sim U(0,1)$ \\
\uIf{$r < \epsilon$}{
$a$ is picked randomly from $A$
}
\uElse{
$A_0 \gets A$ \\
\For{o $= 1, K$}{
  \uIf{$|A_{o-1}| > 1$ \And $o < K$}{
    $A_o \gets $ $\accacts(s, Q, A_{o-1})$
  }
  \uElse{
    $a \gets \argmax_{a' \in A_{o-1}} Q_i(s,a')$ \\
    \Break
  }
}
}
\KwRet $a$
}
}
\caption{ActionSelection}\label{alg:action-selection}
\end{algorithm}
\end{figure}

Note that the structure of this function is the same for both variants of TLQ
and different instantiations of the function differ in how the $\accacts$ subroutine is implemented. $\accacts$
takes the current state $s$, the Q-function to be used, and the actions acceptable to the objectives up to the last one and outputs the actions acceptable to the objectives up to and including the current one. Below, we will describe how $\Pi_{i}$ has been defined in the literature.

 \noindent\emph{Absolute Thresholding:} \citet{gabor1998multi} propose this approach where the actions with values higher than a real number are considered acceptable. Hence, $\Pi_{i}$ is the subset of $\Pi_{i-1}$ for which Q-values are higher than some $\tau$.\\
%
 \emph{Absolute Slacking:} This is the approach taken by \citet{li2019urban} where a slack from the optimal value in that state is determined and each action within that slack is considered acceptable.


\subsection{\!\!\!Shortcomings of Prior TLQ Approaches\hspace*{-10mm}}\label{sec:tlq:issues}
The shortcomings of TLQ depend on the type of task.
We introduce a taxonomy to facilitate the discussion.

\noindent\emph{Constrained/Unconstrained Objective:} Constrained objectives are bounded in quality by their threshold values above which all values are considered the same. Unconstrained objectives do not have a such threshold value. In an LMDP setting, all objectives but the last one are constrained objectives. Recall that a constrained objective is said to be satisfied when it exceeds its threshold.\\
%
  \emph{Terminating Objective:} \looseness-1An objective that either implicitly or explicitly pushes the agent to go to a terminal state of the MDP. More formally, this means discounted cumulative reward for an episode that does not terminate within the horizon is lower than an episode that terminates.\\
\emph{Endpoint Objective:} An endpoint objective is represented by a non-zero reward in the terminal states of the MDP and zero rewards elsewhere.
  We will call an objective that has non-zero rewards in at least one non-terminal state a path objective.

%


This taxonomy also helps seeing the problematic cases and potential solutions. To summarize the empirical demonstrations about the applicability of TLQ for different parts of the problem space:

  \noindent(I) \emph{TLQ successful case studies}: \citet{vamplew2011empirical} show an experiment where the constrained objective is endpoint, unconstrained objective is terminating, and TLQ works. \citet{li2019urban} show this for a case study where the constrained objective is a non-terminating endpoint constraint. Note that these are just demonstrations of empirical performance for case studies that fall into these categories and the papers do not make claims about general instances.

 \noindent(II) \emph{TLQ does not work}: \citet{vamplew2011empirical} show that TLQ does not work when constrained objective is a path objective. To the best of our knowledge, this is the only failure case that has been shown previously in the literature.

\begin{proposition}
\label{proposition:TLQ-Fails}
    In addition to previously identified failure case in \citet{vamplew2011empirical}, TLQ does not work when the constrained objective is a terminating endpoint objective but the unconstrained one is non-terminating.
\end{proposition}

We next demonstrate Proposition~\ref{proposition:TLQ-Fails}'s claim on a maze problem.

\noindent\textbf{MAZE Benchmark.}
We describe scenarios that are common in control tasks yet TLQ fails to solve. To illustrate different issues caused by TLQ, we need an adaptable multiobjective task. Also, limiting ourselves to finite state and action spaces where the tabular version of TLQ can be used simplifies the analysis. Our $\maze$ environment satisfies all of our requirements. Figure~\ref{fig:maze-small} shows an example.

%
%
\begin{wrapfigure}{L}{0.25\textwidth}
  \begin{CenteredBox}
      \begin{myverbatim}
     __________
     |__|G_|__| 2
     |HH|HH|__| 1
     |__|S_|__| 0
      0  1  2
    \end{myverbatim}
\end{CenteredBox}
\caption{A simple maze that demonstrates how TLQ fails to reach the goal state.}\label{fig:maze-small}
\end{wrapfigure}
In all $\maze$ instantiations, each episode starts in $S$ and $G$ is the terminal state. There are also two types of bad tiles. The tiles marked with $HH$ are high penalty whereas the ones with $hh$ show the low penalty ones (Figure~\ref{fig:maze-small} does not have $hh$ tiles). In this work, we will use $-5$ as high penalty and $-4$ as low penalty. But we consider the penalty amounts parameters in the design of a maze; so, they could change. There are two high-level objectives: Reach $G$ and avoid bad tiles. Ordering of these objectives and exact definition of them results in different tasks. We use these different tasks to cover different parts of the problem space described in the taxonomy.
The action space consists of four actions: $\Up, \Down, \Left, \Right$. These actions move the agent in the maze and any action which would cause the agent to leave the grid will leave its position unchanged.


\looseness-1\noindent\textbf{Problems with Endpoint Constraint.} A common scenario in control tasks is having a primary objective to reach a goal state and a secondary objective that evaluates the path taken to there. Formally, this is a scenario where the primary objective is a endpoint objective. However, TLQ either needs to ignore the secondary objective or fails to satisfy the primary objective in such cases. All of the thresholding methods above fail to guarantee to reach the goal in this setting when used threshold/slack values are uniform throughout state space. The maze in Figure~\ref{fig:maze-small} can be used to observe this phenomenon. Assume that our primary objective is to reach $G$ and we encode this with a reward function that is $0$ everywhere but $G$ where it is $R$. And our secondary objective is to avoid the bad tiles.
A Pareto optimal policy in this maze could be indicated by the state trajectory $(1,0) \to (2,0) \to (2,1) \to (2,2) \to (1,2)$. However, this is unattainable by TLQ.

\looseness-1Since the reward is given only in the goal state, $\qshat_1$ can be equal to $\tau_1$ only for the state-action pairs that lead to the goal state. All others will be discounted from these; hence, less than the threshold $\tau_1$. This means the agent will be focusing only on primary objective and ignore the secondary objective when using absolute thresholding of \citep{gabor1998multi}. If no discounting was used, all actions would have the value $\tau_1$, hence the agent would not need to reach the goal state.
We believe the reason why \citet{vamplew2011empirical} have not observed this issue in their experiments with undiscounted ($\gamma = 1$) Deep Sea Treasure (DST) is due to their objectives. The secondary objective of DST, minimizing the time steps before terminal state, is a terminating objective which pushes the agent to actually reach a goal state. Absolute Slacking also similarly leads to contradicting requirements when the trajectory described above is aimed but we leave this case to the reader due to space constraints.

%
%
%
%
%
%

\section{Policy Gradient Approach for TLO}\label{sec:pg}

In this section, we introduce our policy gradient approach that utilizes consecutive gradient projections to solve LMDPs. Policy gradient methods treat the performance of a policy, $J(\theta)$, as a function of policy parameters that needs to be maximized and employ standard
gradient ascent optimization algorithms. Following this modularity, we start with proposing a general optimization algorithm,  the Lexicographic Projection Algorithm (LPA), for multiobjective optimization (MOO) problems with thresholded lexicographic objectives. Then, we will show how single objective Policy Gradient algorithms can be adapted to this optimization algorithm.

As gradients w.r.t. different objectives can be conflicting for MOO, various ways to combine them have been proposed. \citet{UCHIBE20081447} proposes projecting gradients of the less important objectives onto positive halfspaces of the more important gradients. Such a projection vector has the highest directional derivative w.r.t. the less important objective among directions with non-negative derivatives w.r.t. the important objectives. This is assumed to protect the current level for the important objective while trying to improve the less important.
However, a non-negative derivative actually does not guarantee protection of the current level as infinitely small step sizes are not used in practice. For example, for a strictly concave function, the change in a zero-derivative direction will be always negative for any step size greater than $0$.
Therefore, we propose projecting onto \emph{hypercones} which allows more control over
\emph{how safe} the projection is. A hypercone is the set of vectors that make at most a $\frac{\pi}{2} - \Delta$ angle with the axis vector; a positive halfspace is a special hypercone where $\Delta = 0$. Increasing $\Delta$ brings the projection closer to the axis vector. A hypercone $C^{\Delta}_a$ with axis $a \in \Real^n$ and angle $\frac{\pi}{2} - \Delta$ is defined as
\begin{equation}
C^{\Delta}_a =\ 
   \left\{\vv{x} \in \Real^n \left\lvert \norm{\vv{x}} = 0 \lor \frac{\vv{a}^T\vv{x}}{\norm{\vv{a}}\norm{\vv{x}}} \geq \cos{(\frac{\pi}{2} - \Delta)}\right.\right\}
\end{equation}
We can derive the equation for projection of $\vv{g}$ on $C^{\Delta}_a$ by firstly showing that $\vv{g}$, $\vv{a}$, and the projection $\vv{g^p}$ are planar using
Karush-Kuhn-Tucker (KKT) conditions. Then, we can derive the formula by using two-dimensional geometry
giving us $ \vv{g^p} =$
\begin{equation}\label{eq:projection-formula}
  \!\frac{\cos{\Delta}}{\sin{\phi}} \sin{(\Delta+\phi)} \left(\!\vv{g} + \vv{a} \frac{\norm{\vv{g}}}{\norm{\vv{a}}} \paren{\sin{\phi}\tan{\Delta} - \cos{\phi}}\! \right)
\end{equation}
where $\phi$ is the angle between $\vv{a}$ and $\vv{g}$. Moving forward, we assume a function $projectCone(\vv{g}, \vv{a}, \Delta)$ which returns $\vv{g^p}$
according to this equation. Note that when $\vv{g} \in C^{\Delta}_a$, $projectCone(\vv{g}, \vv{a}, \Delta) \vcentcolon= \vv{g}$.

\textbf{Lexicographic Projection Algorithm (LPA).}
A \emph{Thresholded Lexicographic MOO (TLMOO)} problem with $K$ objectives and $n$ parameters can be formulated as maximizing
a function $F:A \to \Real^K$ where $A \subseteq \Real^n$,
and the ordering between two function values $F(\theta_1), F(\theta_2)$ is according to
$\geq^{\tau}$ as defined as in Section~\ref{sec:background}.
Notice that when we have multiple objectives, the gradients will form a $K$-tuple, $G = (\nabla F_1, \nabla F_2, \cdots, \nabla F_K)$, where $\nabla F_i$ is the gradient of $i^{th}$ component of $F$.

Since TLMOO problems impose a strict importance order on the objectives and it is not known how many objectives can be satisfied simultaneously beforehand, a natural approach is to optimize the objectives one-by-one until they reach the threshold values. However, once an objective is satisfied, optimizing the next objective could have a detrimental effect on the satisfied objective. We can use hypercone projection to avoid this. More formally, when optimizing $F_i$, we can project $\nabla F_i$ onto the intersection of $\set{C^{\Delta}_{\nabla F_j}}_{j < i}$ where $\Delta$ is a hyperparameter representing how risk-averse our projection is, and use the projection as the new direction. If such an intersection does not exist, it means that it is not possible to improve $F_i$ without taking a greater risk and we can terminate the algorithm.

We show in Proposition~\ref{prop:alg-monotone-increase} that the LPA algorithm provides a strict increase in the value of the objective it currently optimizes while not decreasing the value of previously optimized objectives until it reaches a point where it cannot improve conservatively. 

\begin{proposition}\label{prop:alg-monotone-increase}
Assume that when applied with a small enough step size on a TLMOO of $F_i$ concave and $L$-smooth for every $i \in \{1\dots K\}$, LPA yields the sequence in parameter space $\{\vv{x}^{(l)}\}_{l \in \Nat}$ with $\vv{x}^{(l)} \in \Real^n \; \forall l \in \Nat$. Define the objective under optimization at step $l$ as $i^*_l = \min_j\{j < K:F_j(\vv{x}^{(l)}) < \tau_j \}$ where $\min \{\} = K$. Then, $F_j(\vv{x}^{(l+1)}) \geq F_j(\vv{x}^{(l)}) $ for all $j \leq i^*_l$ and the relation is strict for $j = i^*_l$ when $\vv{x}^{(l)}$ is not the optimal point of $F_{i^*_l}$ or is a $\Delta$-Pareto-stationary point for the objective functions ${F_1, F_2, \dots, F_{i^*_l}}$. 
\end{proposition}

The proof and the exact lower bounds on $F_j(\vv{x}^{(l+1)}) - F_j(\vv{x}^{(l)})$ in the proposition
can be found in Appendix~\ref{sec:conv-prop}.

Proposition~\ref{prop:alg-monotone-increase} uses the concept of \emph{$\Delta$-Pareto-stationary points} --- a generalization of Pareto-stationary points that allows for parameterized conservativeness. A point $x$ is said to be \emph{$\Delta$-Pareto-stationary} for the objective functions $\{F_1, F_2, \dots, F_k\}$ if there is not an ascent direction $\vv{u}$ such that
$\max_{j \in {1\dots k}} \angle \paren{\vv{u}, \nabla F_j} \leq \frac{\pi}{2} - \Delta$.

\begin{remark}
Under the $\mu$-strong convexity assumptions, a $\Delta$-Pareto-stationary point can be understood as a point where the norm of largest ascent step cannot be greater than $\frac{2\sin(\Delta)}{\mu} \max_j \norm{\nabla F_j(x)}$.
\end{remark}

Improving this approach with a heuristic that prevents overly conservative solutions leads to better performance in certain cases. Conservative updates usually lead to further increases on the already satisfied objectives instead of keeping them at the same level. This means most of the time, we have enough buffer between the current value of the satisfied objectives and their thresholds to sacrifice some of it for further gains in the currently optimized objective. Then, we can define a set of "active constraints" -- a subset of all satisfied objectives -- for which we will not accept any sacrifice and only consider these when projecting the gradient. The "active constraints" can be defined loosely, potentially allowing a hyperparameter that determines the minimum buffer needed to sacrifice from an objective.

%
%

\begin{algorithm}[!h]
{\small
  \SetAlgoLined
\SetKwProg{Fn}{Function}{:}{}
\Fn{\FindDirection{$M, F(\theta), \tau, \Delta, AC, b$}}{
Initialize action-value function $Q$ with random weights\\
\For{o $= 1, K$}{
  \uIf{$o = K$ \Or $F_o(\theta) < \tau_o$}{
  Initialize direction $\vv{u}$ with initial state $M_o$ \\
  \For{j $= 1, o-1$}{
  \uIf{\Not(($AC$ \And $F_j(\theta) > \tau_j + b$) \Or $\angle(\vv{u}, M_j) < \frac{\pi}{2} - \Delta$)}{
  $\vv{u} \gets projectCone(\vv{u}, M_j, \Delta)$
  }
  }
  \For{j $= 1, o$}{
  \uIf{\Not(($j \neq K$ \And $AC$ \And $F_j(\theta) > \tau_j + b$) \Or $\angle(\vv{u}, M_j) < \frac{\pi}{2} - \Delta$)}{
  \KwRet \None
  }
  }
  \KwRet $\vv{u}$
}
}
}
}
\caption{Lexicographic Constrained Ascent Direction}
\label{alg:findDirection}
\vspace{1mm}
\end{algorithm}
%
The $\finddir$ function in Algorithm~\ref{alg:findDirection} (our LPA algorithm)
incorporates these ideas.
This function takes the tuple of all gradients $M$, the tuple of current function values $F(\theta)$, threshold values $\tau$, the conservativeness hyperparameter $\Delta$, a boolean $AC$ that determines whether "active constraints" heuristic will be used or not, and a buffer value $b$ to be used alongside active constraints heuristic as inputs. Then, it outputs the direction that should be followed at this step, which can replace the gradient in a gradient ascent algorithm. For the optimization experiments, we will be using the vanilla gradient ascent algorithm.
Algorithm~\ref{alg:findDirection} finds the first objective that has not passed its threshold and iteratively projects its gradient onto hypercones of all previous objectives. If such a projection exists, it returns the projection as the "Lexicographic Constrained Ascent" direction. Otherwise, it returns null.
In our experiments, we will set $b = 0$. In general, $b$ can be set to any non-negative value and higher values of $b$ would result in a more conservative algorithm which does not sacrifice from an objective unless it is \emph{well} above the threshold.

\textbf{Experiment.}
As a benchmark for the Lexicographic Projection Algorithm, we used the objective functions $F_1(x,y) = -4x^2 + -y^2 +xy$ and $F_2(x,y) = -(x-1)^2 -(y-0.5)^2$ which are taken from \citep{zerbinati:inria-00605423}.
We modified $F_2$ slightly for better visualization and multiplied both functions with $-1$ to convert this to a maximization problem. We set the threshold for $F_1$ to $-0.5$.
The behavior of our cone algorithm without using active constraints heuristic on this problem with $\tau=(-0.5)$ and $\Delta=\frac{\pi}{90}$ can be seen in Figure~\ref{fig:vanilla-cone-values}. 
\begin{figure}
	\centering
	\includegraphics[trim=0 0 20 0, scale=0.55]{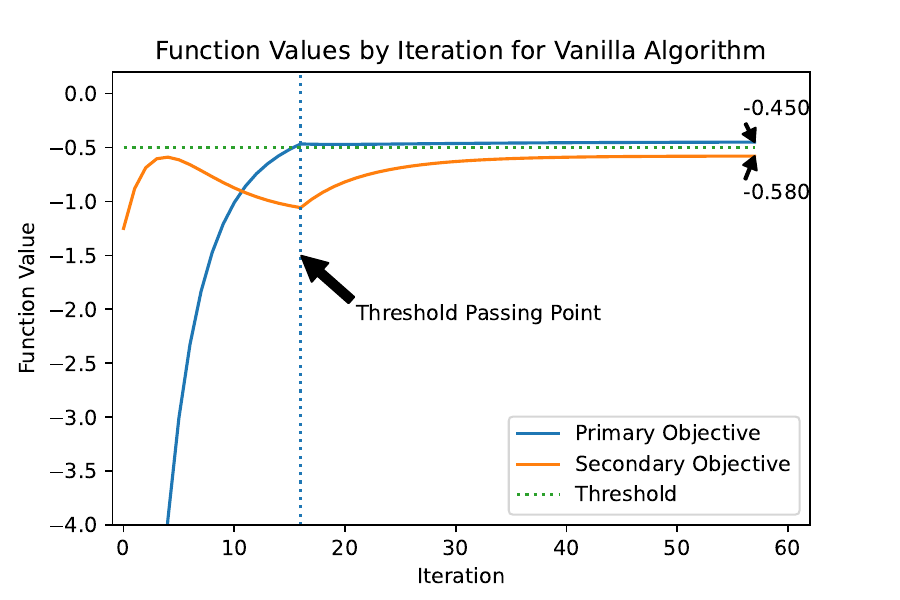}
         \vspace*{3mm}
	\caption{The changes in the function values. Notice that $F_2$, in orange, is ignored until the threshold for $F_1$ is reached. Then, $F_2$ is optimized while respecting the passed threshold of $F_1$.}
	\label{fig:vanilla-cone-values}
        \vspace*{6mm}
      \end{figure}


\textbf{Using Lexicographic Projection Algorithm in RL.}
We show how LPA can be combined with policy gradient algorithms.
We use REINFORCE\citep{NIPS1999_464d828b} as the base policy gradient algorithm because its simplicity minimizes conceptual overhead.
%

We can adapt REINFORCE to work in LMDPs by repeating the gradient computation for each objective independently and computing a new direction using $\finddir$ function. Then, this new direction can be passed to the optimizer. Algorithm~\ref{alg:lexicographic-reinforce} shows the pseudocode for this algorithm.


\begin{algorithm}
  {\small
\SetAlgoLined
\SetKwFunction{reinforce}{REINFORCE}
\SetKwProg{Pr}{Process}{:}{}
\Pr{\reinforce{$\tau, \Delta, AC, b, N_e)$}}{
Initialize policy function $\pi(a|s, \theta)$ with random parameter $\theta$ \\
\For{ep $= 1, N_e$}{
  Generate an episode $S_0, A_0, R_1, \ldots, S_{T-1}, A_{T-1}, R_T$ and save $\ln \pi(A_t|S_t)$ at every step.\\
  $M \gets \emptyset$ \\
  $F \gets 0$ \\
  \For{o $= 1, K$}{
  $G_{T+1} \gets 0$ \\
  \For{t $= T, 1$}{
      $G_t \gets R_t + \gamma G_{t+1}$ \\
      $F_o \gets F_o + R_t$
  }
  $L \gets -\sum_{t=0, T-1} \ln \pi(A_t|S_t) G_{t+1}$ \\
  Compute gradient of $L$ w.r.t. $\theta$ and append it to $M$
  }
$d = \FindDirection(M, F, \tau, \Delta, AC, b)$ \\
Use $d$ as the gradient for the optimizer step to update $\theta$.
}
\KwRet $\pi(a|s, \theta)$  \\
}
}
\caption{Lexicographic REINFORCE}
\label{alg:lexicographic-reinforce}
\end{algorithm}

Note that our algorithm is compatible with most policy gradient algorithms. \citet{UCHIBE20081447} shows how a similar idea is applied to actor-critic family of policy gradient algorithms which reduces the variance in the gradient estimation by using a \emph{critic} network. We believe that more stable policy gradient algorithms like actor-critic methods could further improve the performance of lexicographic projection approach as our algorithm might be sensitive to noise in gradient estimation.


\section{Experiments}\label{sec:experiments}
We evaluate the performance of the Lexicographic REINFORCE algorithm on two Maze problems and a many-objective benchmark from the literature. 
In both experiments, we use a two layer neural network 
for policy function. \CRC{See Appendix~\ref{sec:supp:pg:exp-rl:policy-function} for further details on the policy function.} 

\begin{wrapfigure}{l}{0.25\textwidth}
\vspace*{-1mm}
  \begin{CenteredBox}
    \begin{myverbatim}
_____________
|__|G_|__|__| 4
|__|hh|hh|hh| 3
|__|__|__|__| 2
|HH|HH|HH|__| 1
|S_|__|__|__| 0
 0  1  2  3
   \end{myverbatim}
\end{CenteredBox}
\vspace*{1mm}
\caption{The maze}\label{fig:maze-extended}
\end{wrapfigure}
\textbf{Experiment 1: Path Objective.}
The primary objective is a path objective, i.e. it takes non-zero values in some non-terminal states.
For this, we flip our objectives from the previous maze setting and define our primary objective as minimizing the cost incurred from the bad tiles. $HH$s give $-5$ reward and $hh$s give $-4$ reward. $+1$ reward is given in the terminal state to extend the primary objective to have rewards in both terminal and non-terminal states. The secondary objective is to minimize the time taken to the terminal state. We formalize this by defining our secondary reward function as $0$ in terminal state and $-1$ everywhere else.We use the Maze in Figure~\ref{fig:maze-extended} for this experiment.
%


%
%
\begin{figure}
  \includegraphics[trim=30 50 00 30, scale=0.37]{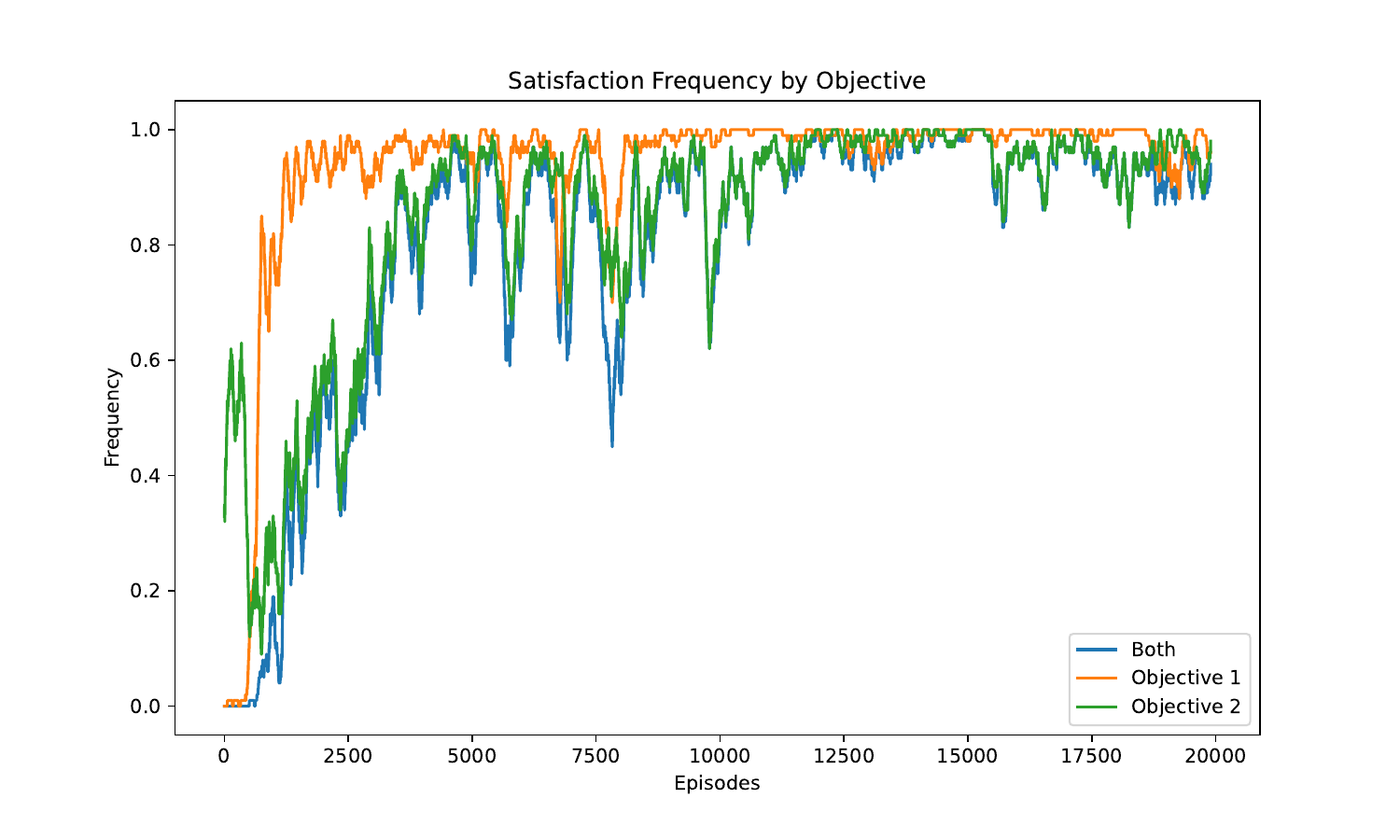}
  \vspace*{3mm}
  \caption{Satisfaction rates for a single successful seed for the path objective maze experiment over 100 episodes.}
    \vspace*{6mm}
  \label{fig:nonreachability-extended-single-seed}
\end{figure}

We found that out of 10 seeds, 7 find policies that have 90\% success over 100 episodes.
The change in satisfaction rates of individual objectives for a successful seed can be seen in Figure~\ref{fig:nonreachability-extended-single-seed}. Notice that the primary objective initially starts very low and quickly increases while the secondary objective does the opposite. Once the primary objective is learned, the algorithm starts improving the secondary.

\textbf{Experiment 2: Endpoint objective (maze).}
As the second experiment, we consider the case where the primary objective is an endpoint objective and the secondary objective is non-terminating,
which was the setting that we found that TLQ fails to reach the goal state in Section~\ref{sec:tlq:issues}. As the agent only cares about eventually reaching the goal, it can completely avoid going on a bad tile.
We run Algorithm~\ref{alg:lexicographic-reinforce} for $N_e = 4000$ episodes.  Out of 10 seeds,
4 find policies that have 90\% success over 100 episodes. Figure~\ref{fig:reachability-simple-singleseed} shows how the satisfaction frequency for each objective changes throughout training. \CRC{Further experimental details and results can be found in Appendix~\ref{sec:supp:pg:exp:reachability}.}

\begin{figure}[tb]
	\centering
	\includegraphics[trim=40 50 0 30, scale=0.37]{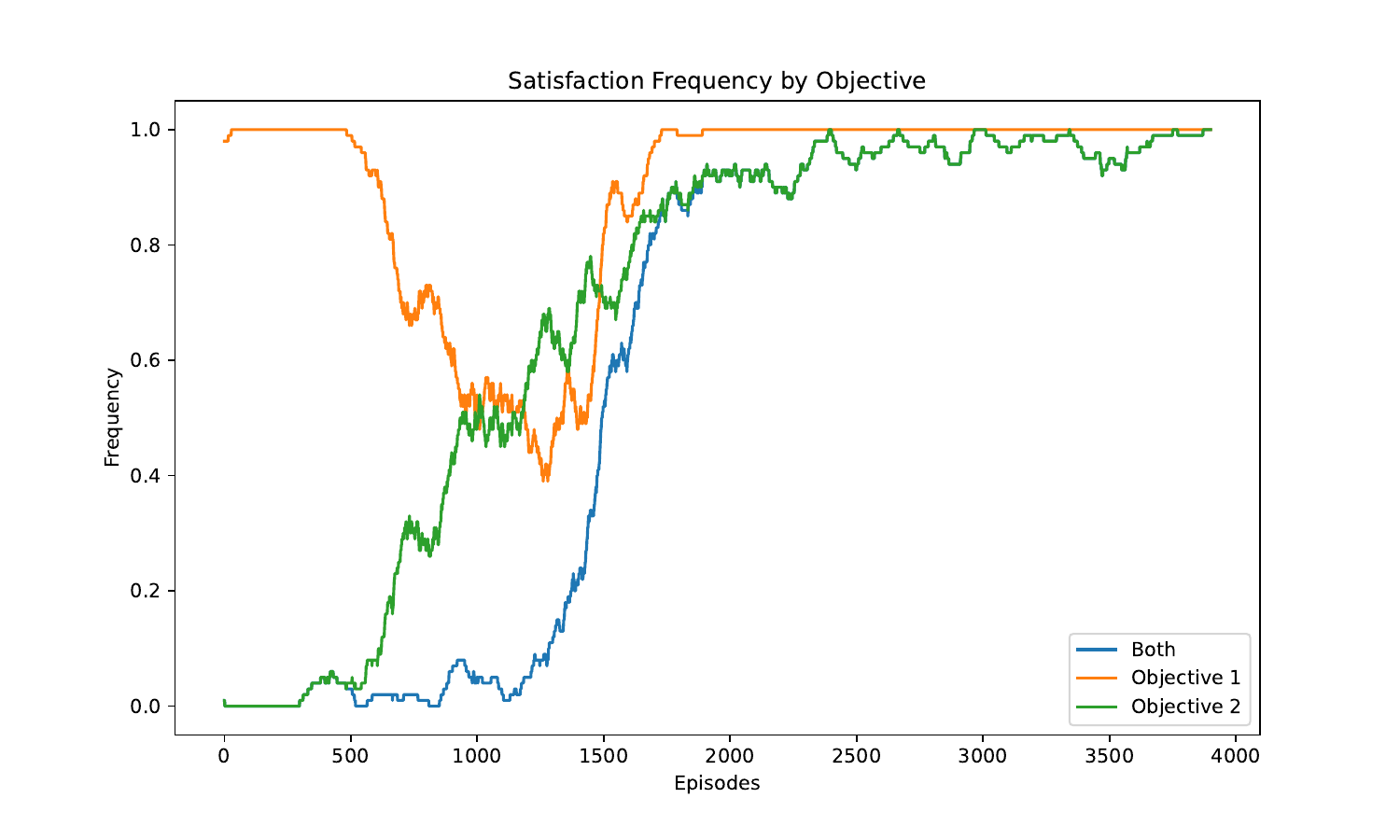}
         \vspace*{3mm}
	\caption{Satisfaction rates for a single successful seed for the endpoint maze experiment over 100 episodes.}
	\label{fig:reachability-simple-singleseed}
  \vspace*{6mm}
\end{figure}

These experiments illustrate the usefulness of projection based policy gradient algorithms for different tasks. We believe our results can be generalized to more complex tasks by combining our algorithm with more stable and sophisticated policy gradient algorithms.

\textbf{Experiment 3: Endpoint Objective (FTN).} Finally, we present our results on Fruit Tree Navigation (FTN) \citep{yang2019generalized}. This task requires the agent to explore a full binary tree of depth $d=5$ with fruits on the leaf nodes. Each fruit has a randomly assigned vectorial reward $\mathbf{r} \in \Real^6$ which encodes the amount of different nutrition components of the fruit. The agent needs to find a path from the root to the fruit that fits to the user preferences by choosing between left and right subtrees at every non-leaf node. As it has non-zero rewards only at the leafs, this is an endpoint objective. 

Figure~\ref{fig:ftn-experiment} compares our algorithm with TLQ on this domain. It also shows the need for the $\Delta$ parameter and highlights that using just hyperplanes as done in \citep{UCHIBE20081447} fails this task. Similarly, TLQ agent also fails to find the desired leaf.


\begin{figure}[tb]
	\centering
	\includegraphics[trim=90 60 0 30, scale=0.34]{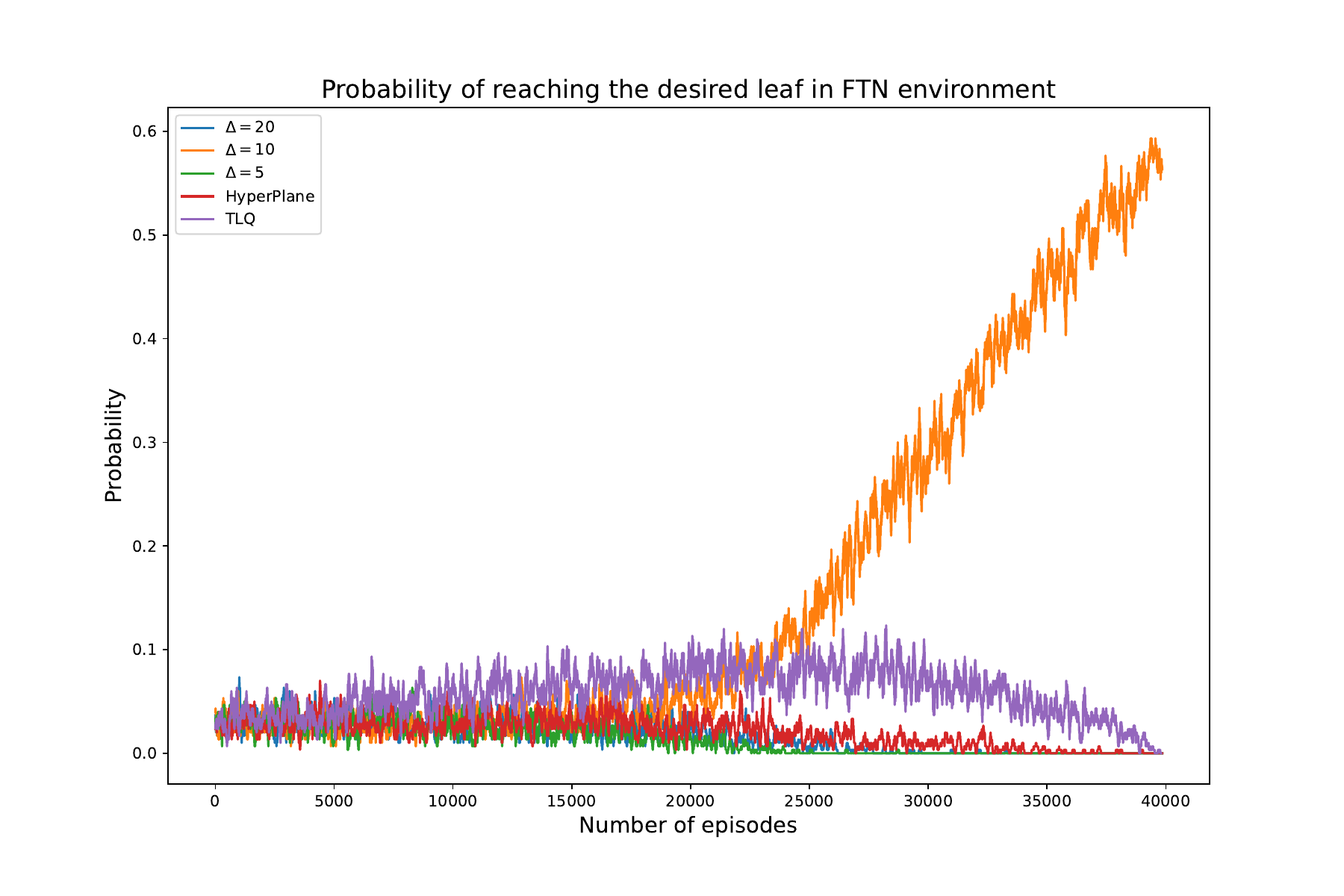}
          \vspace*{3mm}
	\caption{Probability of reaching the desired leaf, averaged over three different random seeds. $\Delta$ values represents conservativeness hyperparameter of the hypercone in degrees.}
\vspace*{7mm}
	\label{fig:ftn-experiment}
\end{figure}

\subsection{Discussions and Comparison with Other Methods}

\CRC{In this section, we will try to summarize the significance of our results w.r.t. other methods from the literature. The first two experiments showcase our method's effectiveness in settings where TLQ is shown to fail. Since TLQ completely fails in these settings, it is omitted from the plots. In Experiment 1, the constrained objective is a path objective which is the setting that \citet{vamplew2011empirical} has found that TLQ fails. However, our experiments show that Lexicographic REINFORCE can actually solve this setting. Similarly, Experiment 2 deals with the setting where the constrained objective is a terminating endpoint objective and the unconstrained objective is non-terminating, a setting we show to be unsolvable by TLQ in Section~\ref{sec:tlq:issues}. Our results show that Lexicographic REINFORCE can solve this problem too.}

\CRC{While Experiment 1 and 2 show that our method does not suffer from the inherent limitations observed in TLQ, they do not show how our method compares to TLQ in a setting that is not picked to highlight the shortcomings of TLQ. For this reason, we test our method on FTN benchmark from the literature. This benchmark has endpoint constraints and the unconstrained objective is terminating, a setting which was shown to be solvable by TLQ (\cite{vamplew2011empirical}). However, our results in Figure~\ref{fig:ftn-experiment} suggest that as the number of objectives increases, TLQ may fail to find an optimal policy. In contrast, Lexicographic REINFORCE can solve this task with the right $\Delta$ parameter. The results also suggest that the hyperplane projection method of \cite{UCHIBE20081447} is not sufficient and the conservativeness of hypercone projection may be needed to successfully converge to the optimal policy in a challenging task. Note that while the performance of Lexicographic REINFORCE is sensitive to $\Delta$ parameter, we believe that combining our method with a more stable policy-gradient algorithm could reduce this sensitivity significantly.}

\CRC{Note that we believe the baselines used in the FTN experiment are the only methods that are directly comparable with our algorithm. In particular, we believe linear scalarization does not offer a fair comparison with our method as it puts a great burden on the user to specify their preferences. Using linear scalarization in this case requires knowing the reward values of all of the fruits beforehand, deciding which fruit would fit the user preference, and finding a weight vector $\omega$ for which the desired fruit is better than others. (Note that such an $\omega$ may not even exist if the desired fruit lies in the concave part of the Pareto front.) In contrast, our method requires only the importance order and the threshold values. However, we also note that if finding such an $\omega$ is not an issue, linear scalarization allows using any method from the vast single-objective RL literature and can yield a great learning performance.}

\section{Conclusion}\label{sec:conclusion}

\looseness-1 In this work, we considered the problem of solving LMDPs using model-free RL. While previous efforts on this problem have focused on value-function based approaches, applicability of these approaches over different problem settings and investigation of 
their shortcomings have been limited. Our first contribution was
providing further insights into the inherent difficulties of developing value function based solutions to LMDPs.
Towards this end, we demonstrated a new significant class of failure scenarios for the existing methods.

Our second focus in this work was developing and presenting a policy-gradient based approach for LMDPs. Policy gradient algorithms have not been studied in MDPs with thresholded lexicographic objectives before even though they are more suitable for this setting
as they bypass many inherent issues with value functions, such as non-convergence due to non-greedy policies w.r.t. value functions, and the need for different threshold values across the state space.
For this, we developed a general thresholded lexicographic multi-objective  optimization procedure based on gradient hypercone projections. Then, we showed how policy gradient algorithms can be adapted to work with LMDPs using our procedure, and demonstrated the performance of our REINFORCE adaptation on three case studies.
While our results are promising for the REINFORCE adaptation,
future research could be further empirical studies with more stable policy-gradient algorithm
adaptations, and over more complex tasks.





\begin{ack}
We thank Sarath Sreedharan, and the conference reviewers; their suggestions and feedback helped us improve the paper.
 This work was supported in part by the National Science Foundation by a CAREER award (grant number 2240126).
\end{ack}


\bibliography{mybibfile}

\onecolumn
\appendix

\newpage
\section{Technical Appendix Organization}

In this section, we will give an overview of the Appendix.
\newline
\\
\textbf{\noindent Technical Appendix for Section \ref{sec:tlq}:} In Appendix~\ref{sec:supp:action-selection}, we share further mathematical details for
the different acceptable policy definitions. In Appendix~\ref{sec:supp:issues}, we firstly
describe how Absolute Slacking and Relative Slacking suffer from the failure to reach the goal.
Then, we present a different issue that was not covered in the main paper: failure to sacrifice at both ends of the episode. After the discussion of issues, we propose different approaches that can address these issues in Appendix~\ref{sec:supp:variations}.
\\
\\
\noindent\textbf{Technical Appendix for Section \ref{sec:pg}:}  We start by describing the derivation of the projection formula used in the paper in Appendix~\ref{sec:supp:cone-projection}. Afterwards, Appendix~\ref{sec:supp:lpa}, we show why cone projection is useful and share the rest of the experiments for the LPA algorithm on the simple optimization benchmark used in the main paper too. Then, in Appendix~\ref{sec:conv-prop}, we analyze the local convergence properties of LPA algorithm and give a proof of Proposition~\ref{prop:alg-monotone-increase}. Finally, we share further results from the experiments with the adapted REINFORCE algorithm in Appendix~\ref{sec:supp:lpa-rl}.
\\
\\
Within this technical appendix, we would like to highlight the following sections.
\\
\\
Appendix~\ref{sec:supp:issues} continues our discussion in Section~\ref{sec:tlq:issues} and demonstrates how existing TLQ variants fail in the given scenarios. Appendix~\ref{sec:supp:variations:informed} presents a TLQ fix that addresses some shortcomings of TLQ for two objectives. Appendix~\ref{sec:supp:variations:augment:multi} formulates how the well-known state augmentation idea can be used to solve LMDPs and proposes this as a new research direction.
\\
\\
For the policy gradient part of our work, Appendix~\ref{sec:supp:pg:justification} illustrates why hypercone projection is needed instead of halfspace projections for lexicographic optimization. Moreover, Appendix~\ref{sec:conv-prop} presents the proof of Proposition~\ref{prop:alg-monotone-increase}. Finally, Appendices \ref{sec:supp:pg:lpa:experiments} and \ref{sec:supp:pg:exp-rl} present rest of the experiments that we could not fit in Section~\ref{sec:pg}.

\section{Further Details on Acceptable Policies}\label{sec:supp:action-selection}

In this section, we will give the mathematical definitions of different thresholding methods. Also, we will describe Relative Slacking, an alternative thresholding method that does not exist in the literature. We include it for the sake of completeness.

\begin{enumerate}
  \item \textbf{Absolute Thresholding:} This is the approach proposed by \cite{gabor1998multi} where the actions with values higher than a real number are considered acceptable. Formally,
  \begin{align}\label{OptimalPolicySetsGabor}
  \Pi_i \triangleq \{ \pi_i \in \Pi_{i-1}\mid\qshat_i(s,\pi_i(s)) = \max_{ a\in \{ \pi_{i-1}(s)\mid\pi_{i-1} \in \Pi_{i-1} \}} \qshat_i(s, a),  \forall s \in \mathcal{S} \}
  \end{align}

  \item \textbf{Absolute Slacking:} This is the approach taken by \cite{li2019urban} and \cite{ijcai2022p476} where a slack from the optimal value in that state is determined and each action within that slack is considered acceptable.
  \begin{align}\label{OptimalPolicySetsUrban}
    \Pi_i \triangleq \{ \pi_i \in \Pi_{i-1}\mid\qstar_i(s,\pi_i(s)) \geq \max_{ a\in \{ \pi_{i-1}(s)\mid\pi_{i-1} \in \Pi_{i-1} \}}  \qstar_i(s,a) - \delta_i,  \forall s \in \mathcal{S} \}
  \end{align}
  Notice that this thresholding scheme is not directly compatible with our definition of LMDPs in Section~\ref{sec:background}. While they are both used to simply introduce some relaxation in policy selection and it does not affect our general analysis, see \cite{wray2015multi} and \cite{pineda2015revisiting} for a definition based on slacks.

  \item \textbf{Relative Slacking:} In this approach, slacks are defined as ratios $\eta \in (0,1]$ rather than absolute values. Then, any action with value greater than $(1 - \eta)$ times the optimal value is considered acceptable. Formally,
  \begin{align}\label{OptimalPolicySetsRelative}
    \Pi_i \triangleq \{ \pi_i \in \Pi_{i-1}\mid\qstar_i(s,\pi_i(s)) \geq (1 - \eta)\max_{ a\in \{ \pi_{i-1}(s)\mid\pi_{i-1} \in \Pi_{i-1} \}}  \qstar_i(s,a),  \forall s \in \mathcal{S} \}
  \end{align}
While has not been proposed in any previous work, we included this for the sake of completeness. Notice that "Relative Thresholding" would be essentially the same technique, only with different parameters.
\end{enumerate}

\section{Issues with TLQ}\label{sec:supp:issues}

In this section, we will elaborate more on the different issues with TLQ.

\subsection{Failing to Reach the Goal}\label{sec:supp:failure-to-reach-goal}

In this section, we will explain how Relative Slacking and Absolute Slacking fail to reach the goal, an issue we discussed in Section~\ref{sec:tlq} for Absolute Thresholding. We will again use Figure~\ref{fig:maze-small}.

Relative Slacking determines the maximum detour. If a non-optimal action delays reaching the goal for $k$ steps, this can be allowed only by defining $\eta > \gamma^k$. However, this detour can be taken repeatedly, preventing actually reaching the goal.

Seeing how Absolute Slacking fails to overcome this problem is a little trickier. It requires a closer inspection of action-values. Since we want that the agent to go left in $(2,2)$, the following should be true:
\begin{align*}
  & \qstar_1((2,2), Right) < \qstar_1((2,2), Left) - \delta_1 \\
   \implies & \gamma R < R - \delta_1 \\
   \implies &\delta_1 < R(1-\gamma)
 \end{align*}

 However, allowing the agent to pick $Right$, instead of $Up$ in $(1,0)$ requires:
 \begin{align*}
   & \qstar_1((1,0), Right) \geq \qstar_1((2,2), Up) - \delta_1 \\
    \implies & \gamma^3 R \geq \gamma R - \delta_1 \\
    \implies & \delta_1 \geq R \gamma (1-\gamma^2)\\
  \end{align*}

  Combining these two requirements implies that:
  \begin{align*}
    & R(1-\gamma) > R \gamma (1 - \gamma^2) \\
    \implies & 1 > \gamma (1+ \gamma) \\
    \implies & 0 > \gamma^2 + \gamma - 1 \qquad \text{Solving the quadratic equation}\\
    \implies & 0.62 > \gamma \\
  \end{align*}

  This shows that to reach the desired policy, not only $\delta$ but $\gamma$ needs to be adjusted too. However, the $\gamma$ parameter is assumed to be an environment constant and traditionally set to values close to $1$. Moreover, there is no real way to find the correct $\gamma$ value apart from computing the action-value function, the very thing we are trying to compute.

Also, a similar analysis shows that small tricks like replacing the primary reward function with
  \begin{align}
  R_1'(s,a,s') = \begin{cases}
    0, \text{if $s' = G$} \\
    -1, \text{otherise}
\end{cases}
  \end{align}
with or without discounting does not solve this problem.

\subsection{Failure to Sacrifice Early and Late}\label{sec:supp:failure-to-sacrifice-middle}

In this section, we will discuss an issue that was not discussed in the main paper: the failure to sacrifice in the early and late parts of the episode. This issue still occurs even if "failure to reach the goal" issue is avoided because the secondary objective happened to be a terminating one.

\begin{figure}[tb]
  \begin{CenteredBox}
      \begin{myverbatim}
        MAZE
     __________
     |__|G_|__| 10
     |HH|HH|__| 9
     |__|__|__| 8
     |__|hh|hh| 7
     |__|__|__| 6
      __|__|__  5
     |__|__|__| 4
     |__|hh|hh| 3
     |__|__|__| 2
     |HH|HH|__| 1
     |__|S_|__| 0
      0  1  2
\end{myverbatim}
\end{CenteredBox}
\vspace*{3mm}
\caption{An example $\maze$ which can be used to demonstrate the issues with uniform thresholding for TLQ.}\label{fig:maze-early-and-late}
\vspace*{6mm}
\end{figure}

Consider the maze shown in Figure~\ref{fig:maze-early-and-late}. {There are bad tiles in four rows and avoiding any of the rows of bad tiles takes two steps. For example, compare the following two paths:
\begin{enumerate}
  \item $(1,0) \to (1,1) \to (1,2)$
  \item $(1,0) \to (2,0) \to (2,1) \to (2,2) \to (1,2)$
\end{enumerate}
Path 2 avoids the bad tiles but it takes 4 steps to get to $(1,2)$ from $(1,0)$ compared to only 2 steps of Path 1.
Since avoiding any tiles costs the same number of extra steps, a natural policy in this maze would be avoiding $HH$ tiles and ignoring $hh$ tiles.
}
 However, this is not possible with either thresholding method. Now, we will discuss how each thresholding method fails this achieving this Pareto optimal policy.

{The action-values show the reward in $G$ discounted by the length of the shortest path to $G$ from the cell this state-action pair leads to. For example, $Q((1,8), Right) = R \gamma^3$ as it takes $3$ steps to get to $G$ from $(2,8)$. Hence, the action-values increase as the agent gets closer to the goal.} Assume that the agent is in $(1,0)$, the action that we need to take is $Right$, meaning $\tau_1$ should be set smaller than or equal to $\gamma^{11} R$ in Absolute Thresholding. However, since the action values will be larger than this in the states closer to $G$, it will mean that the primary objective will be ignored for the rest of the episode. Hence, the agent will avoid $h$ tiles too and the desired policy is unattainable.

Similarly, since Relative Thresholding effectively limits the length of detours and detours for avoiding $h$s are of the same length as the ones for $H$s, this cannot give a policy that only goes through $h$s.

Absolute Slacking will cause this problem in the reverse, meaning the late episode detours requires detours during the whole episode. Assume the agent is in cell $(1,8)$, then we need
\begin{align*}
   & \qstar_1((1,8), R) > \qstar_1((1,8), U) - \delta_1 \\
   \implies & \gamma^3 R > \gamma R - \delta_1 \\
   \implies & \delta_1 > R \gamma(1 - \gamma^2)\\
 \end{align*}

Then, if going left instead of up is not allowed in cell $(1,6)$:
\begin{align*}
&   \qstar_1((1,6), R) < \qstar_1((1,6), U) - \delta_1 \\
\implies & \gamma^5 R < \gamma^3 R - \delta_1 \\
\implies & \delta_1 < R \gamma^3 (1 - \gamma^2)
\end{align*}

Combining these two requirements gives $R \gamma^3 (1 - \gamma^2) > R \gamma(1 - \gamma^2)$, which requires $\gamma > 1$ which is false.
%
%

\section{Variations to TLQ and Some Alternatives}\label{sec:supp:variations}

In this section, we will try to address the problems with TLQ within the framework of value function algorithms. We will start by briefly talking about two of our failed attempts (one completely failing and another half-working) to develop working TLQ variants to show the breadth of the problems and our work. Moreover, we believe these ideas are quite natural and can look promising; so, we would like to share our experience to help people working on TLQ algorithms.

Then, we will describe two of our working solutions. While these solutions are limited either in terms of convergence or applicable domains, they provide either a good solution to a sizeable subset of common tasks or give a good alternative to TLQ in the general case.

\subsection{Failed Attempts}

In this section, we will give our not very successful attempts to improve the performance of TLQ.

\subsubsection{TL-SARSA}\label{sec:tlq:sarsa}

Our first failed attempt was switching to an on-policy learning framework which could solve agents getting stuck problem in Section~\ref{sec:tlq:issues}. {An important reason for this issue was agents' optimistic expectation that they would be following the optimal behavior after each action. So, we considered an on-policy agent which actually uses its realistic behavior to learn could solve our issues.}

So, we modified our update functions from Section~\ref{sec:tlq} to mimic SARSA (\cite{sutton2018reinforcement}) instead of Q-Learning. This would mean replacing the max operators with the actual action. For example, the update function from \cite{li2019urban}
\begin{equation}\label{eq:update-urban}
\qstar_{i}(s, a) = \sum_{s' \in S} P(s,a,s') (R_{i}(s, a, s') + \gamma \max_{\pi \in \Pi_{i-1}} \qstar_{i}(s', \pi(s'))
\end{equation}

will become:
\begin{equation}
\qstar_{i}(s, a) = \sum_{s' \in S} P(s,a,s') (R_{i}(s, a, s') + \gamma \qstar_{i}(s', a')
\end{equation}

where $a' = \pi(s')$.

However, this naive attempt failed due to some theoretical limitations of SARSA. \cite{singh2000convergence} states that the convergence of SARSA is guaranteed under the condition that the policy is greedy in limit. {However, our policies are not necessarily greedy with respect to $\qstar$ in limit. Thresholding means that sometimes actions suboptimal w.r.t. $\qstar$ are chosen. For example, if $\qstar_1(s, a_1) > \qstar_1(s, a_2) > \tau_1$ and $\qstar_2(s, a_2) > \qstar_2(s, a_1)$ for a state $s$ in a two objective task, the policy we want to learn is not greedy w.r.t. $\qstar_1$.} This manifested itself as constant oscillations in the policy in our experiments.

\subsubsection{Cyclic Action Selection}

Our second half-failed attempt was modifying the action selection mechanism to solve the phenomenon described in Section~\ref{sec:supp:failure-to-reach-goal}.
It was based on the intuition that the reason for this issue was unnecessary sacrifices in the primary objective that is not required by the secondary objective. For example, if we consider the maze in Figure~\ref{fig:maze-small}, going left or right in cell $(2,2)$ is the same w.r.t. secondary objective, hence the agent should not sacrifice from the primary objective irrespective of the thresholds/slacks. Using this intuition, we developed a cyclic action selection algorithm. In Algorithm~\ref{alg:cyclic-action-selection}, we show a two objective version of it for simplicity. While it can be generalized to $K$ objectives, we do not believe it would be of interest considering its failure to completely address our problems.

\begin{algorithm}[H]
\SetAlgoLined
\SetKwFunction{CyclicActionSelection}{CyclicActionSelection}
\SetKw{Break}{break}
\SetKwProg{Fn}{Function}{:}{}
\Fn{\CyclicActionSelection{$s, Q | A$}}{
$A_0 \gets A$ \\
$A_1 \gets \accacts(s, Q_1, A_{0})$ \\
\If(\tcp*[f]{If there are not more than one option}){$|A_1| \leq 1$}{
 \KwRet $\argmax_{a \in A_0} Q_1(s,a)$
}
$A_2 \gets \accacts(s, Q_2, A_1)$\\
\If{$|A_2| \leq 1$}{
 \KwRet $\argmax_{a \in A_1} Q_2(s,a)$
}
\KwRet $\argmax_{a \in A_2} Q_1(s,a)$ \tcp*{Max over $A_2$ but wrt $Q_1$}
}
\caption{CyclicActionSelection}\label{alg:cyclic-action-selection}
\end{algorithm}

As the pseudocode shows, the idea is to assign the unconstrained function a threshold/slack which would be used to return the action selection right back to the primary objective but after applying this new threshold. This can be seen from the $Q$-functions used with $\argmax$ and $\accacts$ throughout the algorithm. It starts with using $Q_1$, then uses $Q_2$ if there are multiple acceptable actions w.r.t. $Q_1$. Finally, it uses $Q_1$ again if there is more than one action acceptable w.r.t. $Q_2$. Notice that there is a $\accacts$ call using $Q_2$ which is different than Algorithm~\ref{alg:action-selection}. This requires having a threshold/slacking for the unconstrained objective which is basically a hack. It can be used with any thresholding function from Section~\ref{sec:supp:action-selection}.

However, there are several problems with this approach. Firstly, having a threshold for the unconstrained objective removes one of the important supposed benefits of TLQ, namely its intuitiveness. Especially the cyclic nature of this makes two different threshold values, one for each objective, to be coupled in a complex way when deciding which detours will be taken. This can lead to a blind hyperparameter search. Our experiments show that the success of the policy is highly sensitive to the choices of these two hyperparameters.

Also, the problem described in Section~\ref{sec:supp:failure-to-sacrifice-middle} persists, which means some of the very natural policies cannot be found with this technique.

\subsection{Informed Targets}\label{sec:supp:variations:informed}

{While using the SARSA variant has failed as seen in Section~\ref{sec:tlq:sarsa}, we believe that our intuition about the root of the issues described in Section~\ref{sec:supp:failure-to-reach-goal} was correct. Hence, we decided that an approach that could better align the update target with the "actual policy" could still solve the problem with short-sighted sacrifices. One such way could be accounting for the possibility of actions not being taken according to the given objective. Here, we will present the approach for two objectives. Its generalization to $K$ objectives is not necessarily straightforward and we regard it as a future research direction.
To illustrate the idea, assume that $\qstar_1(s', a_1) > \qstar_1(s', a_2) > \tau_1$ and $\qstar_1(s', a_2) > \qstar_2(s', a_1)$ for a state $s'$ in  a task with only two actions. Eq.~\ref{eq:update-urban} uses $R_1(s,a,s') + \gamma _1(s', a_1)$ when computing update target for $\qstar_1(s,a)$ as $a_1$ maximizes $\qstar_1$ in state $s'$. However, this is
misleading as $a_1$ will never be chosen in state $s'$. Instead $\qstar_1(s',a_2)$ should be used as $a_2$ maximizes $\qstar_2$ in $s'$. Notice that this is still different than TL-SARSA as we may be following a completely different policy. In other words, $a_2$ is used not because it is actually the action taken but it would be the action taken in the optimal case.}
We call this "informed targets" as value functions make "informed" updates, knowing what would be the actual action taken.
More formally, this means modifying the update function for the primary objective to:

{\footnotesize
\begin{align}\label{eq:informed-targets}
\qstar_{1}(s, a)
 = \sum_{s' \in S} (R_{1}(s, a, s') + \qstar_1(s', \argmax_{\pi \in \Pi_1} \qstar_{2}(s', \pi(s')) P(s,a,s')
\end{align}
}

Notice that the target for the objective $1$ is computed by choosing the optimal action with respect to $2$. This prevents optimistic updates that happen due to targets computed with actions that never would be taken.
It should be noted that these updates were the reason for the failure mode discussed in Section~\ref{sec:supp:failure-to-reach-goal}. Preventing them solves this issue but brings a different problem: Instability in update targets. Consider the scenario in Section~\ref{sec:supp:failure-to-reach-goal}. If the current policy is going to left in state $(2,2)$, the value of going right would be $\gamma R$. Assuming that the threshold is smaller than $\gamma R$, at some point the value of going right would pass the threshold and both going right and left would be equally good. Once this happens, the update target for going right will become $\gamma Q(s, Right)$, hence it will start to decrease until it is smaller than the threshold. Then, the target will go back to its original value, hence resulting in an endless cycle. While it is possible to introduce some buffer in these updates such that the oscillations do not affect the policy that is being followed, the optimality of the resulting policy will depend on the initialization.

The update function with buffer hyperparameter $b$ can be obtained by replacing $\Pi_i$ in Section~\ref{eq:informed-targets} with $\hat{\Pi}_i$ which is defined as:

\begin{align}
  \hat{\Pi}_i \triangleq  \{ \pi_i \in \hat{\Pi}_{i-1} \mid \qstar_i(s,\pi_i(s)) \geq \max_{ a\in \{ \pi_{i-1}(s)\mid\pi_{i-1} \in \Pi_{i-1} \}}  \qshat_i(s,a) - \delta_i - b,  \forall s \in \mathcal{S} \}
\end{align}

Notice that this will lead to a smaller oscillation zone which in turn is going to prevent the policy from oscillating as it still uses $\Pi_i$.
Also, note that the problems in Section~\ref{sec:supp:failure-to-sacrifice-middle} still persists.

\subsection{State Augmentation for Non-reachability Constrained Objective Case}

In this section, we will show how a problem where constrained objectives are non-reachability can be solved by augmenting the state space. This idea of state augmentation has been used before with slightly different or narrower purposes \cite{geibel2006reinforcement}.

\subsubsection{Single Constrained Objective}

When the constrained objective is a non-reachability objective, this can be solved by using state augmentation that keeps track of obtained cost/reward for the constrained objective so far. In this section, we will
use a different MDP that is inspired by a real-life scenario to also give a more intuitive example and show the real use of LMDPs.

\textbf{An example:} A car travels across the country using highways. It starts the journey in the city $s_0$ and tries to go to a city $s_F \in S_F$. Once he reaches a city in the set $s_F$, he will stop traveling. The highway toll for the highway from the city $s$ to $s'$ is represented by the function $h(s,s')$ where $h: S \times S \to \realnumbers$. The driver has a budget of $B$ dollars and tries to have the best trip within this budget. His cost within this budget will be reimbursed by his company, so he has no incentive to spend less as long as he is within the budget. His pleasure from arriving in the city $s$ is given by $p(s)$ where $p: S \to \realnumbers$ and $p(s) = 0, \forall s \notin S_F$.

Formally, we have two objectives: minimizing the tolls and maximizing pleasure. Minimizing tolls is constrained/thresholded by the budget $B$. Maximizing the pleasure is unconstrained. Following our formulation in Section~\ref{sec:background}:
\begin{itemize}
\item $R_1(s, a, s') = -h(s,s')$ and $\tau_1 = -B$. Notice again that we expressed the threshold without discounting. Since we will not be using TLQ, we do not need to find the corresponding discounted threshold. Also, notice that the corresponding discounted threshold actually depends on the trajectory.
\item $R_2(s,a,s') = p(s')$.
\end{itemize}

We can express this two-objective task and preserve the preferences by constructing the following single-objective task:

\begin{itemize}
  \item Set of states: $\hat{S} = S \times \realnumbers$ where $(s,c)$ means the driver is in the city $s$ and so far the driver has spent $B-c$ dollars on tolls. Augmented initial state $\hat{s}_0 = (s_0, B)$.
  \item Set of actions: $\hat{A} = A$
  \item Transition function $\hat{P}: \hat{S} \times A \times \hat{S} \to [0,1]$ where
  \begin{equation}
    \hat{P}((s,c), a, (s', c')) =
    \begin{cases}
       P(s,a,s'), c' = c - h(s,s')  \\
       0, \text{otherwise}
     \end{cases}
  \end{equation}
  \item Reward function $\hat{R}: \hat{S} \times A \times \hat{S} \to \realnumbers$
  \begin{equation}\label{equation:Rprime}
    \hat{R}((s,c), a, (s',c')) =
    \begin{cases}
      0, \text{if } s' \notin S_F \\
      p(s'), \text{else if } c' \geq 0 \\
      \lambda  c', \text{otherwise}
    \end{cases}
  \end{equation}
  \end{itemize}

  Note that a non-zero reward will be given only when the car reaches a final destination.
  If the driver has stayed within the budget, he gets his pleasure value as the reward. If he has exceeded the budget, he is penalized accordingly with a multiplier $\lambda$. Implicitly, we assume that $p(s) > 0, \forall s \in S_F$.

  Also, note that while this reward function specifies the optimal policy correctly, it may not be a good reward function for learning and exploration purposes. For instance, until the agent learns how to stay within the budget, all the terminal states will have negative values and non-terminal states will have higher values. Hence, the agent can get stuck here by trying to avoid terminal states. Realizing that it can get positive rewards may require a good exploration policy.

  This has the following advantages:

  \begin{itemize}
    \item Different thresholds are supported
    \item Thresholding is intuitive
    \item Convergence proofs exist.
  \end{itemize}

  \subsubsection{Optimality of New MDP}\label{sec:optimality-of-new-mdp}

  We can easily show that this single-objective task has the same ordering of trajectories as the original task. More formally, $\traject^1 = s^1_0, a^1_0, s^1_1, a^1_1, \dots, s^1_{n^1}$ is better than $\traject^2 = s^2_0, a^2_0, s^2_1, a^2_1, \dots, s^2_{n^2}$ under the original task if and only if the augmented trajectory $\hat{\traject^1}$ is also better than augmented trajectory $\hat{\traject^2}$ under this single-objective task. Here, we will use the cumulative reward as the optimality metric when comparing trajectories. For the original task trajectories, we use the thresholded lexicographic comparison relation defined in Section~\ref{sec:background}. Note that subscripts $n^1$ and $n^2$ denotes the indexes of $\traject^1$
  and $\traject^2$, not the polynomials.

  \paragraph{Proof:} $\traject^1 \geq \traject^2$ under the original task if and only if one of the following must be true:
  \begin{enumerate}
    \item $\sum_{\traject^1} R_1(s,a,s'), \sum_{\traject^2} R_1(s,a,s') \geq
    \tau_1$ and $\sum_{\traject^1} R_2(s,a,s') \geq \sum_{\traject^2}  R_2(s,a,s')$
    \item $\sum_{\traject^1} R_1(s,a,s') \geq \tau_1 > \sum_{\traject^2} R_1(s,a,s')$
    \item $\tau_1 > \sum_{\traject^1} R_1(s,a,s') \geq \sum_{\traject^2} R_1(s,a,s')$
  \end{enumerate}

  We can show that each of these statements implies that the same ordering holds for $\hat{\traject^1} \geq \hat{\traject^2}$ under the single objective task. Firstly observe that:
  \begin{align*}
    & \sum_{\hat{\traject}} \hat{R}(\hat{s},\hat{a},\hat{s}')  = p(\hat{s}_n(s)) \\
    &\iff \sum_{\hat{\traject}} \hat{R}(\hat{s},\hat{a},\hat{s}') > 0 \\
    &\iff \hat{s}_n(s) \in S_F \land \hat{s}_n(c) \geq 0 \\
    & \iff \sum_{\hat{\traject}} h(s,s') \geq B \iff \sum_{\traject} R_1(s,a,s') \geq \tau_1
  \end{align*}

Then, for the first case:

\begin{align*}
  &\sum_{\traject^1} R_1(s,a,s'), \sum_{\traject^2} R_1(s,a,s') \geq
  \tau_1 \\
  \implies& \sum_{\hat{\traject^1}} \hat{R}(\hat{s},\hat{a},\hat{s}') = p(\hat{s}_{n^1}(s)) \land
  \sum_{\hat{\traject^2}} \hat{R}(\hat{s},\hat{a},\hat{s}') = p(\hat{s}_{n^2}(s))
\end{align*}

Also,
\begin{align*}
  &\sum_{\traject^1} R_2(s,a,s') \geq \sum_{\traject^2} R_2(s,a,s') \\ \implies& p(\hat{s}_{n^1}(s)) \geq p(\hat{s}_{n^1}(s)) \\
  \implies& \sum_{\hat{\traject^1}} \hat{R}(\hat{s},\hat{a},\hat{s}') > \sum_{\hat{\traject^2}} \hat{R}(\hat{s},\hat{a},\hat{s}')\\
  \implies&\hat{\traject^1} \geq \hat{\traject^2}
\end{align*}

For the second case,

\begin{align*}
&\sum_{\traject^1} R_1(s,a,s') \geq \tau_1 > \sum_{\traject^2} R_1(s,a,s')\\ \implies& \sum_{\hat{\traject^1}} \hat{R}(\hat{s},\hat{a},\hat{s}') > 0 \land \sum_{\hat{\traject^2}} \hat{R}(\hat{s},\hat{a},\hat{s}') \leq 0 \\
\implies& \sum_{\hat{\traject^1}} \hat{R}(\hat{s},\hat{a},\hat{s}') > \sum_{\hat{\traject^2}} \hat{R}(\hat{s},\hat{a},\hat{s}') \\
\implies& \hat{\traject^1} \geq \hat{\traject^2}
\end{align*}

For the third case, we can observe that:

\begin{align*}
  &\tau_1 > \sum_{\traject^1} R_1(s,a,s') \geq \sum_{\traject^2} R_1(s,a,s') \\
  \implies & B > \sum_{\hat{\traject^1}} h(s,s') \geq \sum_{\hat{\traject^2}} h(s,s') \\
  \implies & \hat{s}_{n^1}(c) \geq \hat{s}_{n^2}(s) \\
  \implies & \lambda \hat{s}_{n^1}(c) \geq \lambda \hat{s}_{n^2}(s) \\
  \implies & \sum_{\hat{\traject^1}} \hat{R}(\hat{s},\hat{a},\hat{s}') > \sum_{\hat{\traject^2}} \hat{R}(\hat{s},\hat{a},\hat{s}') \\
  \implies & \hat{\traject^1} \geq \hat{\traject^2}
\end{align*}

$\blacksquare$
  %

Note that we've specified the unconstrained objective as a quantitative reachability objective, ie. it is non-zero only in the terminal states. Now, we will remove the restriction over $p$.

\paragraph{Alternative 1:} First option is to extend the state space again to keep track of $p$ as well. So, the MDP will be:
\begin{itemize}
  \item State Space: $\doublehat{S} = S \times \realnumbers \times \realnumbers$ where the state $(s, c, \bar{p})$ corresponds to accumulating $\bar{p}$ $p(s)$ so far.
  \item Transition function: $\doublehat{P}: \doublehat{S} \times A \times \doublehat{S} \to [0,1]$ where
  \begin{align}
    \doublehat{P}((s,c, \bar{p}), &a, (s', c', \bar{p}')) \\ &=
    \begin{cases}
       \hat{P}((s,c),a,(s',c')), \bar{p}' = \bar{p} + p(s')  \\
       0, \text{otherwise}
     \end{cases}
  \end{align}

  \item Reward function: $\doublehat{R}: \doublehat{S} \times A \times \doublehat{S} \to \realnumbers$ where
  \begin{equation}
    \doublehat{R}((s,c,\bar{p}), a, (s',c',\bar{p}')) =
    \begin{cases}
      0, \text{if } s' \notin S_F \\
      \bar{p}', \text{else if } c' \geq 0 \\
      \lambda  c', \text{otherwise}
    \end{cases}
  \end{equation}
\end{itemize}

\paragraph{Alternative 2:} Extending the state space is not always optimal, as it increases the complexity. Instead, we can try to directly modify \ref{equation:Rprime}. With this, we will still use $\hat{S}$ and $\hat{P}$ as the state space and transition function, respectively.

\begin{itemize}
  \item Most simply, we can start giving $p(s')$ reward in the non-terminal states. Then, we can guarantee the lexicographic ordering by subtracting a large value $C_l$ that is guaranteed to be larger than $\sum_t{p(s')}$ from $\lambda c'$. \begin{equation}
      \doublehat{R}((s,c), a, (s',c')) =
      \begin{cases}
        p(s'), \text{if } s' \notin S_F \\
        p(s'), \text{else if } c' \geq 0 \\
        \lambda  c' - C_l, \text{otherwise}
      \end{cases}
    \end{equation}
\end{itemize}

Optimality proofs of these new MDPs are very similar to our proof in Section~\ref{sec:optimality-of-new-mdp}. So, we leave it to the reader to avoid repeating it.

%
%
%
%
%
%

\subsubsection{Multiple Constrained Objectives Case}\label{sec:supp:variations:augment:multi}

Our analysis above assumes that there are only two objectives: a constrained primary objective and an unconstrained secondary objective. However, in that setting, many CMDP algorithms are readily applicable. Therefore, we are more interested when we have multiple constraints that need to be solved in the lexicographic order. Yet, extending the approach above to this setting is not straightforward. To apply it, we need to know which constraints can be satisfied together. {More formally, if the constrained objectives are $1, \dots, (k-1)$, we need to find the maximum $i$ such that there exists a policy that satisfies objectives $1, \dots, (i-1)$, i.e. can reach a state $(s,c_1,c_2,\ldots, c_{i-1})$ such that $s \in S_F$ and $c_1,c_2,\ldots, c_{i-1} \geq 0$.} We identified three different approaches that could be used for this but we believe future work is needed to develop more efficient methods.

\paragraph{One-by-one} The simplest method to solve tasks with multiple constrained objectives is reminiscent of linear search algorithm. {We can start with the first (most important) constraint and see if we can find a policy that satisfies it, i.e. can reach $(s, c_1)$ such that $s \in S_F$ and $c_1 \geq 0$ from $s_{init}$. If such a policy exists, we can introduce the second constraint to see if a policy that satisfies both of them simultaneously exists.} Continuing in this fashion, it can be found up to which objective the agent can satisfy simultaneously. However, this method can be prohibitively expensive as it requires solving $O(k)$ subproblems. More importantly, it is very hard if not impossible to know whether a subproblem is not solvable or just taking too long to learn.

For this method, we can construct the following reward function for each $i$ value in different ways. An approach would be maximizing the worst violated constraint:

\begin{align}\label{eq:multi-state-aug-subproblem}
  \hat{R}((s,c_1, \dots, c_{i-1}),& a, (s',c'_1, \dots, c'_{i-1})) = \\ & \begin{cases}
    R(s, a, s'), \text{if } s' \notin S_F \\
    R(s, a, s'), \text{else if } c'_j \geq 0 \;\; \forall j < i \\
    \lambda  \min_j{c'_j} - C_l, \text{otherwise}
  \end{cases}
\end{align}

Where $R$ is the reward function of the unconstrained objective in the original MDP and
$C_l$ is an upper bound on the unconstrained reward that can be collected during an episode.

\paragraph{Binary Search} As the name suggests, this method is inspired by binary search algorithm. Assuming the constrained objectives are $1, \ldots, (k-1)$, we can start by trying to solve constraints $1, \ldots, \lfloor\frac{k}{2}\rfloor$, then we can try $1, \ldots, \lfloor \frac{3k}{4}\rfloor$ or $1, \ldots, \lfloor \frac{k}{4}\rfloor$ depending on whether it was solvable or not, respectively. While this method is faster than one-by-one, it still suffers from the same halting problem. We can use Eq.~\ref{eq:multi-state-aug-subproblem} for this approach too.

\paragraph{Dynamic Search} This method is not concretized and is intended mostly as an idea for future research. \cite{hayes2020dynamic} presents an approach to set the threshold values for TLQ dynamically, depending on the attainable performance up to that point in the training. Similarly, we can introduce and remove constraints dynamically during the training without waiting for the algorithm to successfully converge for a subproblem.

%

\section{Cone Projection}\label{sec:supp:cone-projection}
In this section, we show how the projection equation in Eq.~\ref{eq:projection-formula} is derived. For the sake of completeness, we start with some simpler and well-known projections and move to the derivation of Eq.~\ref{eq:projection-formula}.

\subsection{Orthogonal Projection onto a Hyperplane}

One of the most well-known projection tasks is projecting a vector $\vv{y} \in \Real^n$ onto a hyperplane $H_{\vv{a}}$ that passes through origin, specified by its normal vector
$\vv{a} \in \Real^n$ as $H_{\vv{a}} = \set{\vv{x} \in \Real^n \mid \ang{\vv{x},\vv{a}} = 0}$ where $\ang{}$ denotes the dot product defined as $\vv{v}^T\vv{a} = \sum_i \vv{v}_i \vv{a}_i$. Projection of $\vv{y}$ onto $H_{\vv{a}}$ is notated as $\projp{y}{a}$ and defined as
$\projp{y}{a} = \argmin_{\vv{v} \in H_{\vv{a}}} \norm{\vv{v} - \projp{y}{a}}$. $\norm{}$ denotes the L2 norm defined as
\begin{equation*}
\norm{\vv{v}} = \sqrt{\vv{v}^T\vv{v}} = \sqrt{\sum_i \vv{v}_i^2}
\end{equation*}

$\projp{y}{a}$ can be found easily by using well-known result $\vv{y} - \projp{y}{a} \parallel \vv{a}$, i.e. the projection error is parallel to the normal vector of the hyperplane. Then, there is a $c \in \Real$ such that $\vv{y} - \projp{y}{a} = c \vv{a}$.

\begin{align*}
  & \vv{y} - \projp{y}{a} \parallel \vv{a} \\
  \implies & \projp{y}{a} = \vv{y} - c\vv{a} \\
  \implies & \ang{\projp{y}{a},\vv{a}} = \ang{\vv{y}, \vv{a}} - c\ang{\vv{a}, \vv{a}} \\
  \implies & 0 = \ang{\vv{y}, \vv{a}} - c\norm{\vv{a}}^2 \\
  \implies & c = \frac{\ang{\vv{y}, \vv{a}}}{\norm{\vv{a}}^2} \\
  \implies & \projp{y}{a} = \vv{y} - \frac{\ang{\vv{y}, \vv{a}}}{\norm{\vv{a}}^2} \vv{a}
\end{align*}

\subsection{Projection onto a Halfspace}\label{sec:pg:background:half-space}

In many cases, we may want to not project a vector that is already on one side of the hyperplane. For example, if we want to project a vector onto a feasible set, the vector that is already in the feasible set should not be projected. This idea can be formalized by extending the definition above to halfspaces. A positive halfspace $S^+_{\vv{a}}$ is defined as $S^+_{\vv{a}} = \set{\vv{x} \in \Real^n \mid \ang{\vv{x},\vv{a}} \geq 0}$. This can be thought of as the set of vectors with which $\vv{a}$ makes an angle less than or equal to $\frac{\pi}{2}$. We can define the projection $\vv{y}$ onto $S^+_{\vv{a}}$ as follows:
\begin{align}
  \projs{a}{y} =
  \begin{cases}
  \vv{y},& \vv{y} \in S^+_{\vv{a}} \\
  \projp{a}{y},& \text{otherwise}
  \end{cases}
\end{align}

Note that the piecewise function handles $\vv{y} \in S^+_{\vv{a}}$ and $\vv{y} \not\in S^+_{\vv{a}}$ cases separately.

\subsection{Projecting a vector onto a cone}

While halfspaces are one of the most common sets in practice, they can be limiting in many cases. A natural extension to this idea would be limiting the set to vectors with which $\vv{a}$ makes an angle $\frac{\pi}{2} - \Delta$ for some $0 \leq \Delta \leq \frac{\pi}{2}$.
The angle between two vectors is defined using dot product:

\begin{equation*}
  \ang{v,u} = \cos{\angle{v,u}} \norm{\vv{a}}\norm{\vv{x}}
\end{equation*}
Note that since $\cos(\Delta) = \cos{2\pi - \Delta}$, the $\angle{v,u}$ can take
two values between $0$ and $2\pi$. For simplicity, we will always talk about the
smaller angle, i.e. $\angle: \Real^n \times \Real^n \to [0,pi]$.

This would be a hypercone which simplifies to a halfspace when $\Delta = 0$.

Let $C^{\Delta}_a$ be a hypercone with axis $a \in \Real^n$ and angle $\frac{\pi}{2} - \Delta$, i.e.
\begin{equation}
C^{\Delta}_a = \set{\vv{x} \in \Real^n \mid \norm{\vv{x}} = 0 \lor \frac{\vv{a}^T\vv{x}}{\norm{\vv{a}}\norm{\vv{x}}} \geq \cos{(\frac{\pi}{2} - \Delta)}}
\end{equation}

which uses the dot product formula above to see if cosine of the angle between $\vv{a}$ and $\vv{x}$ is greater than cosine of $\frac{\pi}{2} - \Delta$. For $0 \leq \Delta \leq \frac{\pi}{2}$, this corresponds to the angle between $\vv{a}$ and $\vv{x}$ being in the interval $[0, \frac{\pi}{2} - \Delta]$.

Then, the projection of a vector $\vv{g} \in \Real^n$ onto $C$ is defined as
\begin{equation}
\vv{g}^p_C = \argmin_{\hat{\vv{g}} \in C}{\norm{\hat{\vv{g}} - \vv{g}}_2}
\end{equation}

Solving this equation is not as straightforward as for the halfspaces. We will first show that $\vv{g}^p$ is planar with $\vv{g}$ and $\vv{a}$, i.e. they can be written as linear combinations of each other. are all on the same plane. This is intuitive and well-known in lower dimensions, but below can be seen a formal proof for higher dimensions. Once this is proven, we can utilize some two-dimensional geometric intuition to simplify the algebra.

\subsection{Proof of Planarity}

The projection is a constrained optimization problem:
\begin{align*}
  &\min{\norm{\vv{x} - \vv{g}}_2} \\
  &\textrm{subject to} \quad \frac{\vv{a}^T\vv{x}}{\norm{\vv{a}}\norm{\vv{x}}} \geq \cos{(\frac{\pi}{2} - \Delta)}
\end{align*}

If we can show that the solution to this vector is planar with $\vv{g}$ and $\vv{a}$, we will be done. The solution to this constrained optimization problem should satisfy Karush-Kuhn-Tucker (KKT) conditions which generalize the Lagrange Multiplier method to problems with inequality constraints. However, applying KKT conditions in this format does not provide a clean result. Therefore, we will prove a stronger claim that gives cleaner KKT conditions:

\begin{lemma}
For any fixed length $\vv{x}$, the projection is minimized when $\vv{x}$, $\vv{g}$, and $\vv{a}$ are planar.
\end{lemma}

\begin{proof}
This gives us the following modified optimization problem with an additional constraint. Now, we will show that the planarity does not depend on $\norm{\vv{x}}$, which will be denoted as $c$.

\begin{align*}
  &\min && f(\vv{x}) = \norm{\vv{x} - \vv{g}}_2 \\
  &\textrm{subject to} \quad &&r(\vv{x}) = \frac{\vv{a}^T\vv{x}}{\norm{\vv{a}}\norm{\vv{x}}} \geq \cos{(\frac{\pi}{2} - \Delta)} = \sin{\Delta}\\
  & && h(\vv{x}) = \norm{\vv{x}} = c
\end{align*}

Swapping norms with their dot product equivalents (replacing the norm in the objective and equality constraint with a norm square for conciseness) and writing the remaining in the standard format gives us:

\begin{align*}
  &\min && f(\vv{x}) = \vv{x}^T \vv{x} - 2 \vv{g}^T \vv{x} + \vv{g}^T \vv{g} \\
  &\textrm{subject to} \quad &&r(\vv{x}) =  \sin{\Delta} - \frac{\vv{a}^T\vv{x}}{\sqrt{\vv{a}^T \vv{a}}\sqrt{\vv{x}^T \vv{x}}} \leq 0\\
  & && h(\vv{x}) = \vv{x}^T \vv{x} - c^2 = 0
\end{align*}

KKT conditions for this problem require that any minimum point $\vv{\hat{x}}$ should satisfy the following condition \cite{chong2004introduction}:

\begin{align*}
  &\nabla f(\vv{\hat{x}}) + \lambda \nabla h(\vv{\hat{x}}) + \mu \nabla r(\vv{\hat{x}}) = \vv{0} \\
  \implies &2\vv{\hat{x}} - 2\vv{g} + \lambda 2\vv{\hat{x}} + \mu \frac{\vv{a}(\sqrt{\vv{a}^T \vv{a}}\sqrt{\vv{\hat{x}}^T \vv{\hat{x}}}) - \frac{1}{2}\frac{\sqrt{\vv{a}^T \vv{a}}}{\sqrt{\vv{\hat{x}}^T \vv{\hat{x}}}} 2\vv{\hat{x}}(\vv{a}^T\vv{\hat{x}})}{(\vv{\hat{x}}^T \vv{\hat{x}})(\vv{a}^T \vv{a})} = \vv{0} &&& \vv{\hat{x}}^T \vv{\hat{x}} = c^2 \quad \text{from feasibility}\\
  \implies & 2\vv{\hat{x}} - 2\vv{g} + \lambda 2\vv{\hat{x}} + \mu \frac{\vv{a}(c\sqrt{\vv{a}^T \vv{a}}) -\frac{\sqrt{\vv{a}^T \vv{a}}}{c}\vv{\hat{x}}(\vv{a}^T\vv{\hat{x}})}{(c^2)(\vv{a}^T \vv{a})} = \vv{0} \\
  \implies & 2\vv{\hat{x}} - 2\vv{g} + \lambda 2\vv{\hat{x}} + \mu \frac{c\vv{a} -\frac{\vv{a}^T\vv{\hat{x}}}{c}\vv{\hat{x}}}{c^2 \sqrt{\vv{a}^T \vv{a}}} = \vv{0} &&& \frac{\vv{a}^T\vv{\hat{x}}}{\sqrt{\vv{\hat{x}}^T \vv{\hat{x}}}} = \sin{\Delta} \sqrt{\vv{a}^T \vv{a}} \quad \text{C.S.} \\
  \implies & 2\vv{\hat{x}} - 2\vv{g} + \lambda 2\vv{\hat{x}} + \mu \frac{c\vv{a} -\sin{\Delta}\sqrt{\vv{a}^T \vv{a}}\vv{\hat{x}}}{c^2 \sqrt{\vv{a}^T \vv{a}}} = \vv{0} \\
  \implies & 2\vv{\hat{x}} - 2\vv{g} + \lambda 2\vv{\hat{x}} + \mu \frac{\vv{a}}{c\sqrt{\vv{a}^T \vv{a}}} -\mu \frac{\sin{\Delta}}{c^2}\vv{\hat{x}} = \vv{0} &&& \text{Reorganize the terms}  \\
  \implies & \vv{\hat{x}}(2 + 2\lambda -\mu \frac{\sin{\Delta}}{c^2}) = 2\vv{g} - \mu \frac{\vv{a}}{c\sqrt{\vv{a}^T \vv{a}}}
\end{align*}

Since $\vv{g}^p$ is such a minimum point, the above analysis holds for it too. Hence, it can be written as a linear combination of $\vv{a}$ and $\vv{g}$. This means that the three vectors are planar.

\end{proof}

Note that we can see another important result from the analysis above. The complementary slackness condition of KKT requires that $\mu r(\vv{\hat{x}}) = 0$. However, if $\mu = 0$, the last line equation in the proof simplifies to
\begin{equation*}
  \vv{\hat{x}}(2 + 2\lambda) = 2\vv{g}
\end{equation*}
If $(2 + 2\lambda) \geq 0$, it means $\vv{g}$ and $\vv{\hat{x}}$ in the same direction. This is only possible if $\vv{g}$ is already in the hypercone.
If $(2 + 2\lambda) < 0$, this means $\vv{g}$ and $\vv{\hat{x}}$ are in the opposite
directions which cannot be the projection, as choosing $0$ vector would give a
smaller projection error.
Hence, unless $\vv{g}$ is already in the hypercone and does not require a projection, $r(\vv{\hat{x}})$ should be $0$. That means the angle between $\vv{a}$
and $\vv{\hat{x}}$ is $\frac{\pi}{2} - \Delta$.

\subsection{Derivation of the Projection Formula}

Now that it is known that all three vectors are planar, we can just use two-dimensional geometry to reason about it and derive the formula. This can be done as these three vectors in $\Real^n$ will span a two-dimensional subspace of $\Real^n$ unless they are all collinear, i.e. scalar multiplicative of each other. This would mean that $\vv{a}$ and $\vv{g}$ are already in the same direction and no projection is needed, which is a special case we will consider separately. Also, any 2-dimensional subspace of $\Real^n$ is isomorphic to $\Real^2$, i.e. identical in structure\cite{linear-algebra-book.ch5}. Figure~\ref{fig:geometry-of-projection} shows the case when the angle between $\vv{a}$ and $\vv{g}$, $\phi$, is larger than $\frac{\pi}{2}$. It can be confirmed that the other configurations like will result in the same equations too. Note that when writing the equations below, we considered when $\vv{g}$ is outside of the cone. When $\vv{g} \in C$, we will simply call $\vv{g}^p = \vv{g}$ similar to the piecewise function in Section~\ref{sec:pg:background:half-space}.

\begin{figure}[tb]
	\centering
	\includegraphics[trim= 40 0 40 0, scale=0.3]{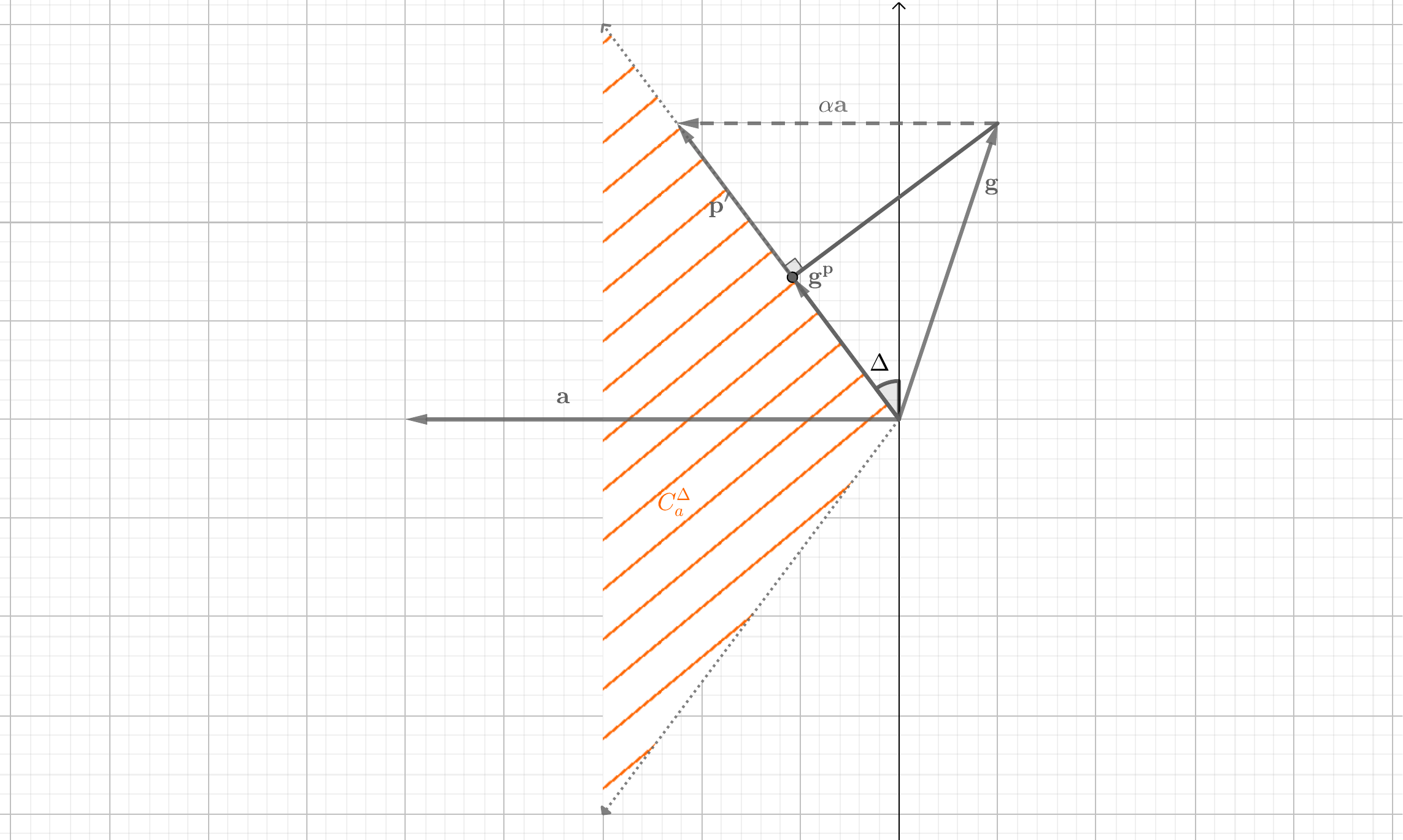}
                                \vspace*{3mm}
	\caption{This figure shows how the vectors would be positioned on a plane. The orange region shows the cone. The angle between $\vv{a}$ and $\vv{g}$, $\phi$, is omitted to not crowd the figure.}
        \vspace*{6mm}
	\label{fig:geometry-of-projection}
\end{figure}

Firstly, we will find the direction of the projection. Let $\vv{p'}$ be a vector with the same direction as $\vv{g}^p$ and it can be written as below.
\begin{align*}
  \vv{p'} = \vv{g} + \alpha \vv{a}
\end{align*}

Then, we can find $\alpha$ as below by using the law of sines:
  \begin{align*}
  & \alpha \norm{\vv{a}} = \norm{\vv{g}} \sin{(\phi - \frac{\pi}{2})} + \norm{\vv{g}} \cos{(\phi - \frac{\pi}{2})} \frac{1}{\sin{(\frac{\pi}{2} - \Delta)}} \sin{\Delta} \\
  \implies &   \alpha \norm{\vv{a}} = -\norm{\vv{g}} \sin{(\frac{\pi}{2} - \phi)} + \norm{\vv{g}} \cos{(\frac{\pi}{2}-\phi)} \frac{1}{\sin{(\frac{\pi}{2} - \Delta)}} \sin{\Delta} \\
  \implies &  \alpha \norm{\vv{a}} = -\norm{\vv{g}} \cos{\phi} + \norm{\vv{g}} \sin{\phi} \frac{1}{\cos{\Delta}} \sin{\Delta} \\
  \implies &  \alpha \norm{\vv{a}} = \norm{\vv{g}} (\sin{\phi} \frac{1}{\cos{\Delta}} \sin{\Delta} - \cos{\phi})\\
  \implies &  \alpha = \frac{\norm{\vv{g}}}{\norm{\vv{a}}} (\sin{\phi} \frac{1}{\cos{\Delta}} \sin{\Delta} - \cos{\phi}) \\
  \implies &  \alpha = \frac{\norm{\vv{g}}}{\norm{\vv{a}}} (\sin{\phi}\tan{\Delta} - \cos{\phi}) \\
  \implies &  \vv{p'} = \vv{g} +  \frac{\norm{\vv{g}}}{\norm{\vv{a}}} (\sin{\phi}\tan{\Delta} - \cos{\phi})\vv{a}
\end{align*}

This $\vv{p'}$ has the correct direction but not necessarily the correct norm to minimize the projection error. The correct projection will be $\vv{g}^p = k\vv{p'}$ where $k \in \Real$. We can find the
$k$ using the well-known rule that the projection error is perpendicular to the projection.
\begin{align*}
&\ang{\vv{g} - k\vv{p'}, \vv{p'}} = 0 \\ 
\implies& \ang{\vv{g}, \vv{p'}} - k\ang{\vv{p'}, \vv{p'}} = 0\\
\implies& \norm{\vv{g}}\norm{\vv{p'}}\cos{(\Delta + \phi - \frac{\pi}{2})} - k\norm{\vv{p'}}^2 = 0 \\
\implies& \norm{\vv{g}}\norm{\vv{p'}}\cos{(\frac{\pi}{2}-\Delta + \phi)} - k\norm{\vv{p'}}^2 = 0 \\
\implies& \norm{\vv{g}}\norm{\vv{p'}}\sin{(\Delta + \phi)} - k\norm{\vv{p'}}^2 = 0 \\
\implies& \norm{\vv{p'}}(\norm{\vv{g}}\sin{(\Delta + \phi)} - k\norm{\vv{p'}}) = 0 \\
\implies& \norm{\vv{g}}\sin{(\Delta + \phi)} - k\norm{\vv{p'}} = 0 \qquad \text{if $\norm{\vv{p'}} \neq 0$}\\
\implies& \norm{\vv{g}}\sin{(\Delta + \phi)} = k\norm{\vv{p'}} \\
\implies& k = \frac{\norm{\vv{g}}}{\norm{\vv{p'}}}\sin{(\Delta + \phi)}
\end{align*}

The same result also could be obtained by solving another optimization problem with $k$ as the variable. Combining the formula for $\vv{p'}$ and $k$ gives the
formula for $\vv{g}^p$.sec:supp:pg:exp:reachability

Moving forward, we'll assume a function $projectCone(\vv{g}, \vv{a}, \Delta)$ which returns the projection of $\vv{g}$ onto $C^{\Delta}_{\vv{a}}$ possibly handling $\vv{g} \in C^{\Delta}_{\vv{a}}$ and $\vv{g} \notin C^{\Delta}_{\vv{a}}$ cases separately.

\section{Lexicographic Projection Algorithm}\label{sec:supp:lpa}

In this section, we start by giving some background on gradients and directional derivatives that is necessary to understand our algorithm. Then, we share the formulation of the lexicographic optimization problems we solve. Finally, we give some further justification on why we need cone projection instead of halfspace projection and share the remaining results with our algorithm that was left out of the main paper due to space constraints.
\subsection{Background on Gradient and Directional Derivatives}

Gradient of a function gives the direction and rate of the fastest increase from point $p$. Moreover, directional derivative of $F$ at $p$ along direction $\vv{u}$, i.e. $\frac{\partial F}{\partial \vv{u}}(p)$, can be computed as $\ang{\vv{u}, \nabla F(p)}$.

Intuitively, the directional derivatives give the rate of change $\nabla F(p)$ in the given direction. As can the dot product implies, this rate is the largest when the angle between $\vv{u}$ and $\nabla F$ is zero. In other words, the gradient gives the direction of the fastest increase.

Using directional derivatives, we can reason about how changes to $p$ affect the value of $F$. For example, since the gradient has the fastest instantaneous rate of change, $F(p + \epsilon \frac{\nabla F(p)}{\norm{\nabla F(p)}}) \geq  F(p + \epsilon \frac{\vv{u}}{\norm{\vv{u}}}), \forall \vv{u} \in \Real^n$ for sufficiently small $\epsilon$.

Similarly, if $\angle{\vv{u}, \nabla F(p)} \leq \frac{\pi}{2}$, $F(p + \epsilon \vv{u}) \geq F(p)$ for sufficiently small $\epsilon$. This can be confirmed by computing the directional derivative using $\ang{\vv{u}, \nabla F(p)} = \norm{u}\norm{\nabla F(p)} \cos{\angle{\vv{u}, \nabla F(p)}}$. In other words, using the directional derivatives, we can obtain a direction of non-decrease for a sufficiently small step size.

%
%
%

\subsection{Formulation of Thresholded Lexicographic Multi-Objective Optimization Problems}

A generic multi-objective optimization problem  with $K$ objectives and $n$ parameters can be formulated as:

Given a function $F:A \to \Real^K$ where $A \in \Real^n$ and a comparison relation $\geq^c$ for value vectors in $\Real^K$, find an element $\theta^* \in A$ such that $f(\theta^*) \geq^c f(\theta)$ for all $\theta \in A$.


Notice that when we have multiple objectives, the gradients will form a $K$-tuple, $G = (\nabla F_1, \nabla F_2, \cdots, \nabla F_K)$, where $\nabla F_i$ is the gradient of $i^{th}$ component of $F$.

Different instantiations of the comparison relation lead to various multi-objective problem families. In the case of Lexicographic Multi-Objective Optimization, the comparison relation $>^c$, is defined as

\begin{align*}
  \vv{v_1} >^c \vv{v_2}  \iff \exists i < K \;\; \text{s. t.} & \; \forall j < i \; \; \vv{v_1}(j) \geq \vv{v_2}(j)) \; \\
  & \land \vv{v_1}(i) > \vv{v_2}(i))
\end{align*}

In \emph{Thresholded} Lexicographic Multi-Objective Optimization, a threshold vector $\tau \in \Real^{K-1}$ is introduced to express the values after which the user does not care about improvements in that objective. This new comparison relation can be denoted by $>^{(c, \tau)}$ which is defined as:

$ \mathbf{u} >^{\vv{\tau}} \mathbf{v}$ iff there exists $ i \leq K$ such that:
\begin{itemize}
\item $\forall j < i $ we have $\mathbf{u_j} \geq \min(\mathbf{v_j}, \tau_j)$; and
  \begin{itemize}
\item
  if $i<K$ then $\min(\mathbf{u_i}, \tau_i) > \min(\mathbf{v_i}, \tau_i)$,

\item otherwise if $i=K$ then $\mathbf{u_i} > \mathbf{v_i}$.
\end{itemize}
\end{itemize}

The relation $ \geq^{\vv{\tau}} $ is defined as $ >^{\vv{\tau}} \, \vee \, =$.

Notice that this completely parallels the definition of LMDPs from Section~\ref{sec:background}.

\subsection{Justification of Cone Projection}\label{sec:supp:pg:justification}

Since thresholded lexicographic multi-objective optimization problems impose a strict importance order on the objectives and it is not known how many objectives can be satisfied simultaneously beforehand, a natural approach is to optimize the objectives one-by-one until they reach the threshold values. However, once an objective is satisfied, optimizing the next objective could have a detrimental effect on the satisfied objective. This could even lead to the previous objective failing. While we can always go back to optimizing this failing objective, this would be inefficient, even worse, potentially leading to endless loops of switching between objectives.

However, we could limit our search for a satisfying point for the new objective to the directions not detrimental to already satisfied objectives by using our results about directional derivatives. For simplicity, assume that we have a primary objective $F_1$ which is satisfied at the current point $\theta_n$ and a secondary objective $F_2$ which we are trying to optimize next. $\nabla_{\vv{u}} F_1$, the change in $F_1$ along a direction $\vv{u}$, is equal to $\ang{\vv{u}, \nabla F_1(\theta_n)} = \norm{\vv{u}}\norm{\nabla F_1(\theta_n)} \cos \paren{\angle{\vv{u}, \nabla F_1(\theta_n)}}$, choosing a direction which makes an angle $\phi \in [-\frac{\pi}{2}, \frac{\pi}{2}]$
with $\nabla F_1(\theta_n)$ would make the directional derivative non-negative. Therefore, updating $\theta$ as $\theta_{n+1} = \theta_n + \epsilon \vv{u}$ with an infinitesimal $\epsilon$ would not reduce the value of $F_1$. If  $\nabla_{\vv{u}} F_2$ is positive, we can optimize $F_2$ without jeopardizing $F_1$. Note that the same logic hold even if we have $k$ already satisfied objectives $F_1, \ldots, F_k$ and now optimizing $F_{k+1}$ as long as  $\forall i \leq k \nabla_{\vv{u}} F_i \geq 0$.

While any such $\vv{u}$ allows us to carefully optimize our new objective $F_2$, we should pick an $\vv{u}$ with maximum $\frac{\partial F_2}{\partial \vv{u}}(\theta_n)$ to optimize $F_2$ most efficiently. While we know that $\nabla F_2(\theta_n)$ has the maximum directional derivative, it may not satisfy our previous requirements.  Instead, we can use the vector projection to find the $\vv{u}$ which minimizes $\norm{\vv{u} - \nabla F_2(\theta_n)}$ under the constraint $\nabla_{\vv{u}} F_1 \geq 0$.
Notice that non-negative directional derivative means that $\vv{u}$ lies on the positive halfspace of $\nabla F_1(\theta_n)$, i.e. $\vv{u} \in S^+_{\vv{u}}$.
So, projecting $\nabla F_2(\theta_n)$ onto $ S^+_{\nabla F_1(\theta_n)}$ will give us the $\vv{u}$ which satisfies the requirement and is closest to $\nabla F_2{\theta_n}$, i.e. has the largest directional derivative. As a special case, when $\nabla F_1$ and $\nabla F_2$ point in opposite directions, this projection will give a zero vector which means that we cannot optimize $F_2$ without sacrificing $F_1$. This point would be locally Pareto optimal. In general, iteratively projecting $\nabla F_{k+1}(\theta_n)$ on the positive halfspaces of $\nabla F_1(\theta_n), \ldots, \nabla F_k(\theta_n)$
gives the desired vector as long as the final vector satisfies the requirements. If it does not satisfy the requirements, this point can be called a locally Pareto optimal point.

While the approach above has the theoretical guarantees for the infinitely small step size, this does not translate to practice as the step sizes are not small enough. For example, Figure~\ref{fig:tangent-fail} shows how a direction that lies on the positive halfspace of the gradient can lead to a decrease. It can be also seen that unless the step size is infinitely small, this would always lead to a decrease. We can overcome this issue by generalizing halfspace to a hypercone for which the central angle is $\frac{\pi}{2} - \Delta$ where $\Delta$ is the hyperparameter of conservativeness. For $\Delta = 0$, this would be the halfspace case introduced above. Figure~\ref{fig:how-cone} shows how hypercone projection differs from halfspace projection and keeps the function above or at the current level for reasonably large step sizes.

\begin{figure}[tb]
	\centering
	\includegraphics[trim= 300 140 300 0, scale=0.35]{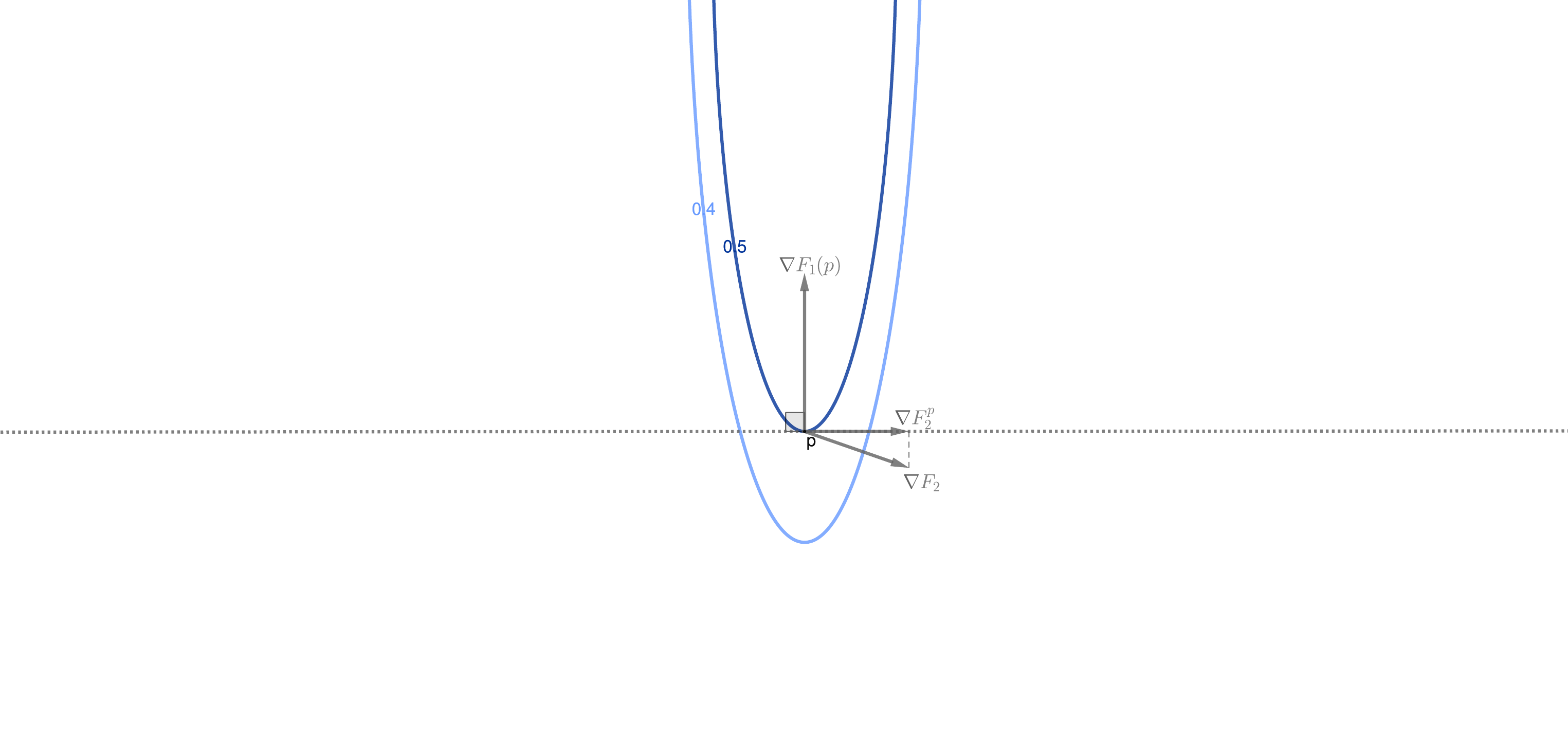}
                        \vspace*{3mm}
	\caption{This figure shows why projecting onto the positive hyperspace is not always enough. The curves show the level curves of a function $F_1$. $\nabla F_2^p$ shows the projection of $\nabla F_2$ on the positive halfspace of $\nabla F_1$. Note that following $\nabla F_2^p$ reduces the function from $0.5$ to $0.4$.}
                \vspace*{6mm}
                \label{fig:tangent-fail}
\end{figure}

\begin{figure}[tb]
	\centering
	\includegraphics[trim= 200 0 240 0, scale=0.35]{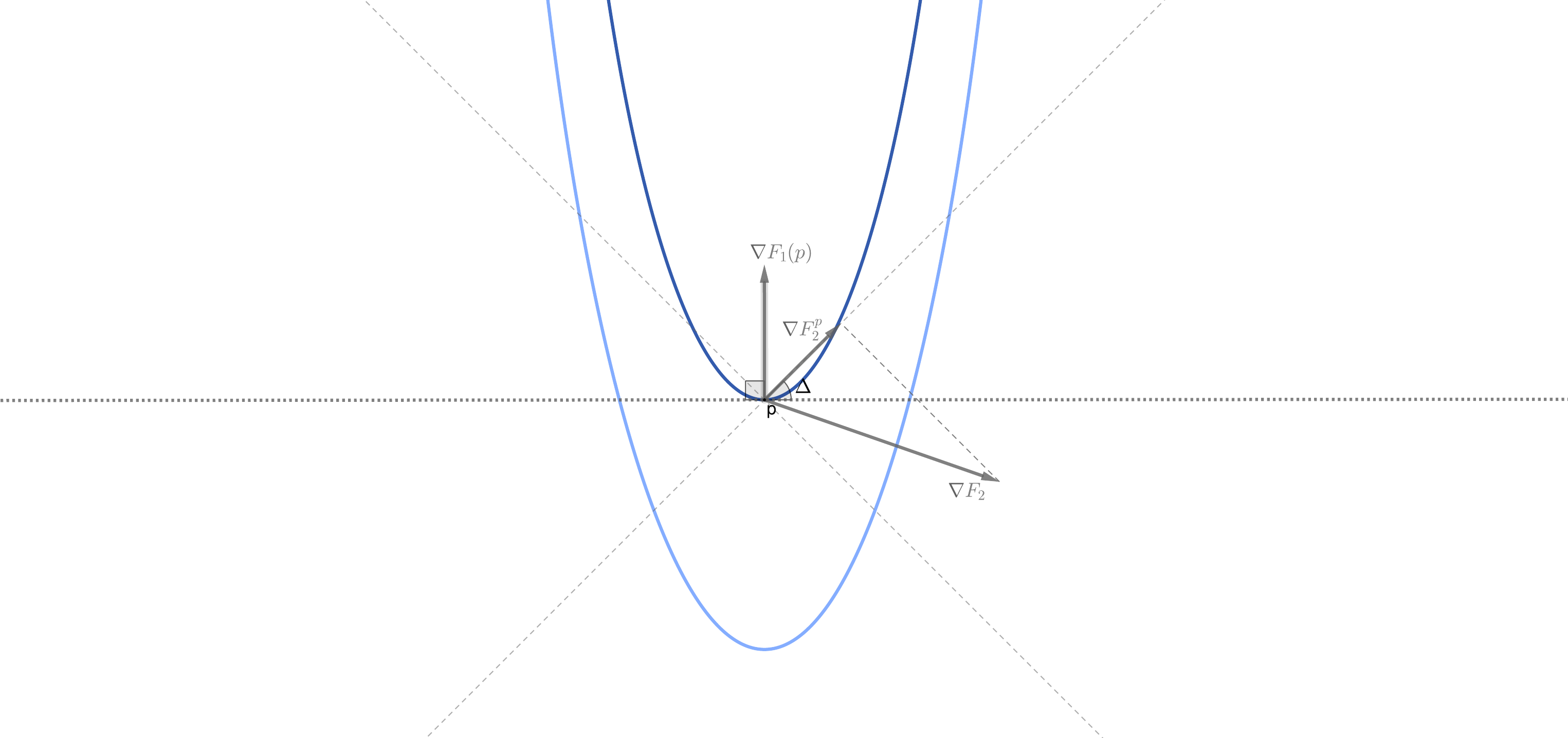}
                                \vspace*{3mm}
                                \caption{A visualization of cone projection. The dashed line shows the boundaries of the cone which are two lines in two dimensions. Notice that following $\nabla F_2^p$ keeps the function from at or above $0.5$ unless a very large step size is chosen.}
                                                        \vspace*{6mm}
	\label{fig:how-cone}
\end{figure}

\subsection{Experiments}\label{sec:supp:pg:lpa:experiments}

In this section, we present the rest of the results for the experiments shown in Figure~\ref{fig:vanilla-cone-values}. All of the experiments are done using the same benchmark problem described in the main paper.

Figure~\ref{fig:vanilla-cone-trajectory} demonstrates how the value changes shown in Figure~\ref{fig:vanilla-cone-values} were reflected on the parameter space. The figure shows the trajectory the algorithm takes over the level curves of the functions. Notice that $F_2$, in blue, is completely ignored until the threshold for $F_1$ is reached. Then, the algorithm optimizes $F_2$ while respecting the passed threshold of $F_1$, indicated by its trajectory almost along the level curve of $F_1$.

\begin{figure}[tb]
	\centering
	\includegraphics[scale=0.4]{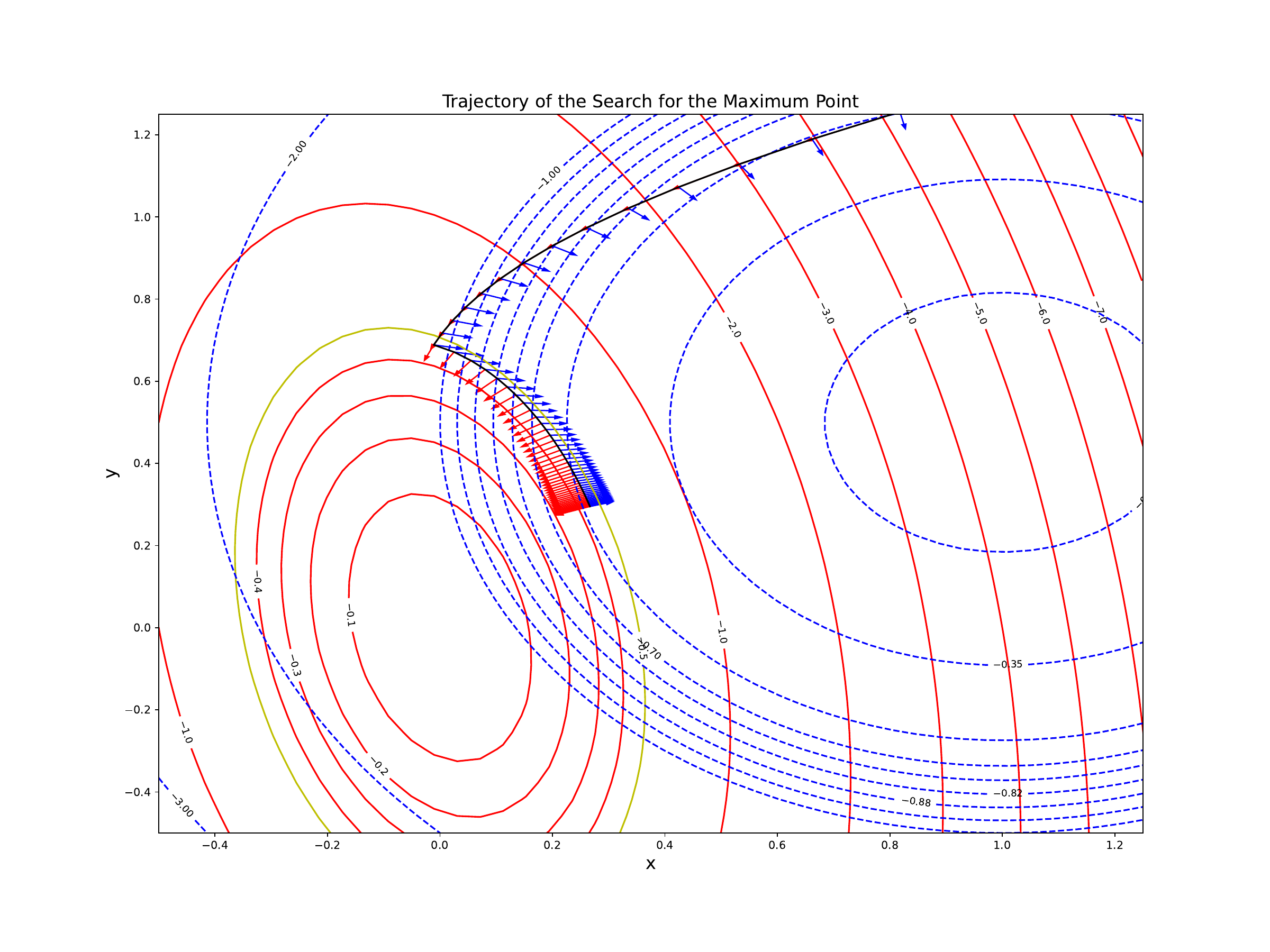}
                                \vspace*{3mm}
                                \caption{Behavior of the algorithm without the active constraints heuristic and hyperparameters $\alpha=0.2$ and $\Delta = \frac{\pi}{90}$. The red and blue curves show the level curves of $F_1$ and $F_2$, respectively. The single yellow curve shows the threshold for $F_1$. The black line shows the trajectory of the solution, while the red and blue arrows show the gradients w.r.t. $F_1$ and $F_2$, respectively.}
                                                        \vspace*{6mm}
	\label{fig:vanilla-cone-trajectory}
\end{figure}

Repeating the same experiment with $AC$ heuristic and $b=0.01$ yields the results shown in Figure~\ref{fig:ac-cone-trajectory} and Figure~\ref{fig:ac-cone-values}. Notice that the highest value we were able to obtain for $F_2$ was $-0.580$ without $AC$ heuristic; but this was improved to $-0.554$ with $AC$. This is because $AC$ prevents unnecessarily improving $F_1$ over the threshold. This can be observed from the final values of $F_1$ which is $-0.450$ without $AC$ and $-0.496$ with $AC$. The downside is losing the smooth and safe trajectory allowed by our vanilla algorithm, indicated by the zig-zags in Figure~\ref{fig:ac-cone-trajectory} and Figure~\ref{fig:ac-cone-values}. The zig-zags represent the corrections for sacrificing too much from $F_1$ when optimizing $F_2$.


\begin{figure}[tb]
	\centering
	\includegraphics[scale=0.35]{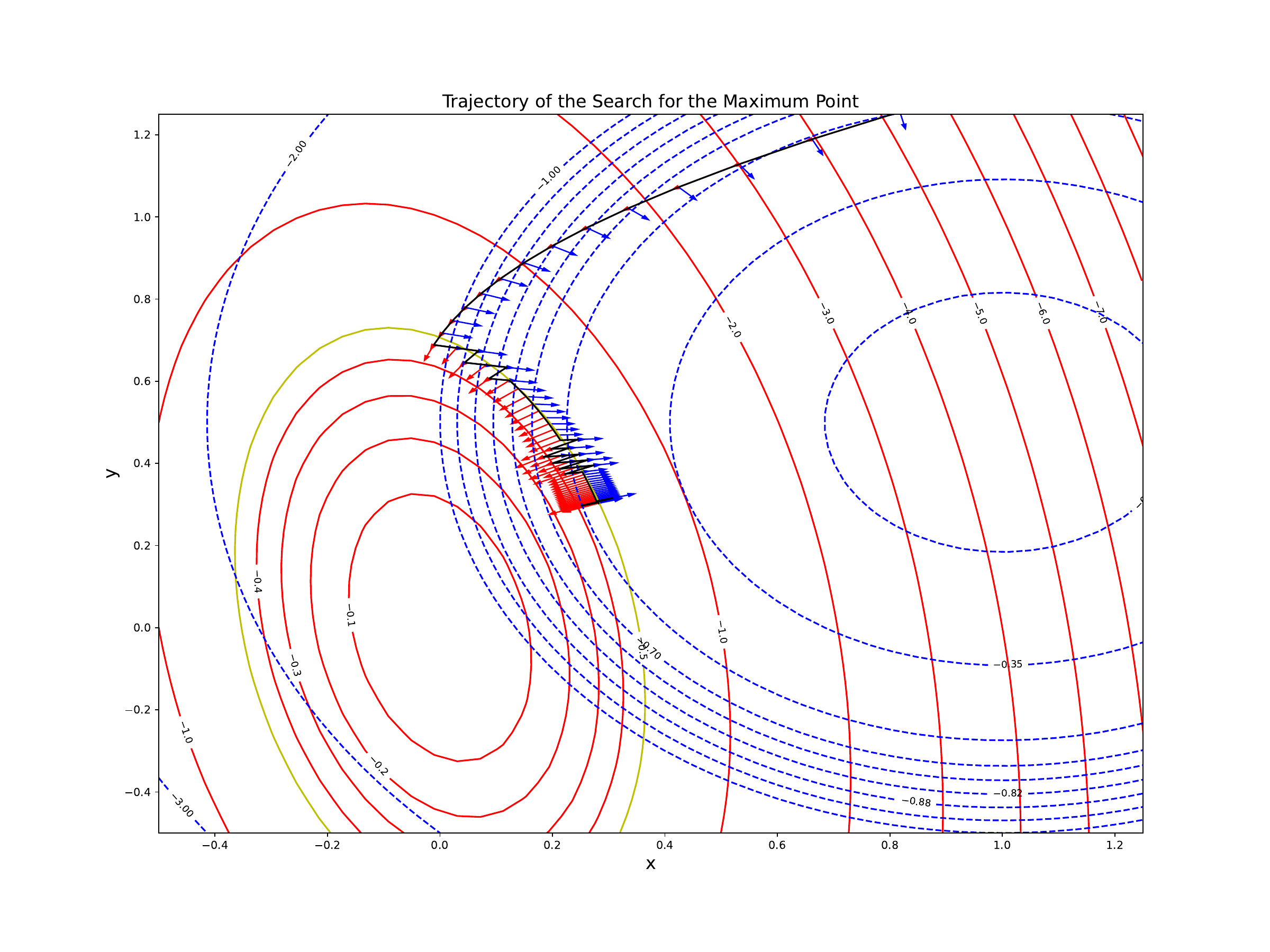}
                                \vspace*{3mm}
	\caption{Behavior of the algorithm with the active constraints heuristic and $b=0.01$. The rest of the hyperparameters are as described in Figure~\ref{fig:vanilla-cone-trajectory}.}
                                \vspace*{6mm}
	\label{fig:ac-cone-trajectory}
\end{figure}

\begin{figure}[tb]
	\centering
	\includegraphics[scale=0.6]{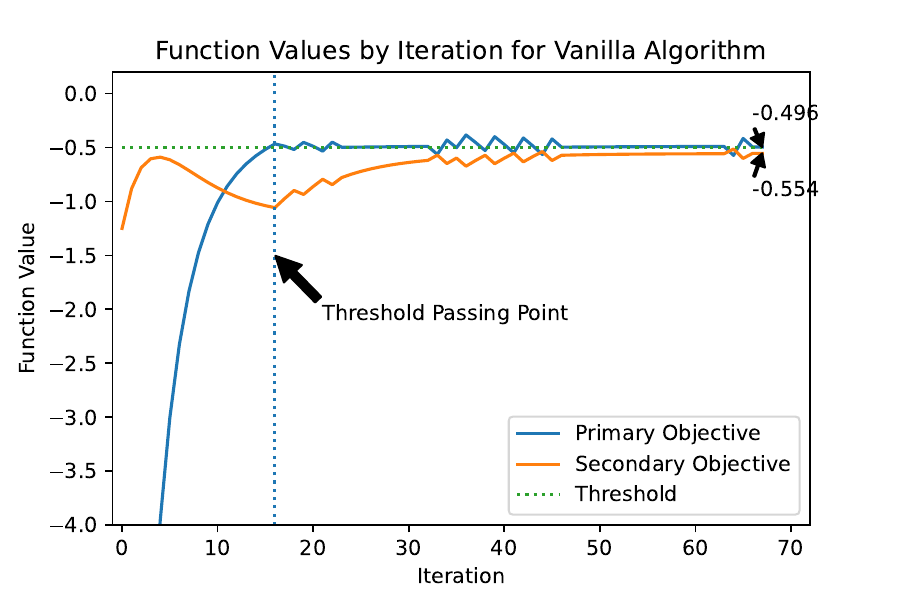}
                                \vspace*{3mm}
                                \caption{The changes in the function values for the experiment described in Figure~\ref{fig:ac-cone-trajectory}.}
                                                        \vspace*{6mm}
	\label{fig:ac-cone-values}
\end{figure}

\section{Convergence Properties}\label{sec:conv-prop}

In this section, we will study the convergence properties of our algorithm and prove Proposition~\ref{prop:alg-monotone-increase}. We will denote the a 2-norm of a vector with $\norm{}$ and assume that the objective functions are concave and $L$-smooth. This allows us utilizing Lemma~\ref{lemma:concave_ub}.

\begin{lemma}\label{lemma:concave_ub}
    For a concave and $L$-smooth function $f: \Real^n \to \Real$ following inequality is satisfied for any $\vx, \vv{y} \in \Real^n$:
    \begin{equation}
        f(\vv{y} ) \geq f(\vx) + \nabla f(\vx)^T (\vv{y} -\vx) - \frac{L}{2} \norm{\vv{y} -\vx}^2
    \end{equation}
\end{lemma}

\subsection{Two objective case}\label{sec:conv-prop:two-obj-case}
Firstly, we will consider the two objective case, i.e. $K=2$. Then, applying our algorithm results in a sequence in the parameter space $\{\vx^{(l)}\}_{l \in \Nat}$ where $\vx^{(l)} = \vx^{(l-1)} + \alpha \vu$.
\begin{equation}
    \vu = \begin{cases}
        \nabla f_1(\vx) \;\; & \textrm{if} \; f_1(\vx) < \tau_1 \\
        \projc{\nabla f_2(\vx)}{\nabla f_1(\vx)} \;\; & \textrm{otherwise}
    \end{cases}
\end{equation}


We can see that our algorithm has two modes which changes the computation of gradient step $\vu$: $f_1(\vx) < \tau_1$ and $f_1(\vx) \geq \tau_1$. We will inspect these two modes separately. When we are in the first mode, i.e. $f_1(\vx) < \tau_1$, the algorithm reduces to  fixed step-size gradient ascent. Hence, the known properties of gradient descent will hold. In particular, 

\begin{lemma}
    For concave and $L$-smooth $f_1$, $f_1(\vx^{l+1}) \geq f_1(\vx^i) + \frac{\alpha}{2}\norm{\nabla f_1(\vx^i)}^2$ for all $i \in \Nat$.
\end{lemma}

\begin{proof}
\begin{align*}
    & f_1(y) \geq f_1(\vx) + \nabla f_1(\vx)^T (y-\vx) - \frac{L}{2} \norm{y-\vx}^2 && \text{(a)}\\
    \implies &f_1(\vx^+ \geq f_1(\vx) + \nabla f_1(\vx)^T (\alpha \vu) - \frac{L}{2} \norm{\alpha \vu}^2 && \text{(b)}\\
    \implies & f_1(\vx^+) \geq f_1(\vx) + \nabla f_1(\vx)^T (\alpha \nabla f_1(\vx) - \frac{L}{2} \norm{\alpha \nabla f_1(\vx)}^2 && \text{(c)}\\
    \implies & f_1(\vx^+) \geq f_1(\vx) + \alpha \nabla f_1(\vx)^T\nabla f_1(\vx) - \frac{L}{2} \norm{\alpha \nabla f_1(\vx)}^2\\
    \implies & f_1(\vx^+) \geq f_1(\vx) + \alpha\norm{\nabla f_1(\vx)}^2 - \frac{L\alpha}{2} \norm{\nabla f_1(\vx)}^2\\
    \implies & f_1(\vx^+) \geq f_1(\vx) + \alpha(1 - \frac{L\alpha}{2}) \norm{\nabla f_1(\vx)}^2 && \text{(d)}\\
\end{align*}

(a) follows from Lemma~\ref{lemma:concave_ub}, (b) comes from taking $y=\vx^{(l+1)}, \vx=\vx^{(l)}$ and denoting $\vx^+ = \vx^{(l+1)}$, $\vx = \vx^{(l)}$ for ease of notation. (c) is obtained by plugging in the value of $\vu$ in the first mode. (d) shows that if we pick a small enough step size such that $\alpha \leq \frac{2}{L}$, ${f_1(\vx^{l})}_{l\in \Nat}$ is a strictly increasing sequence unless $\nabla f_1(\vx) = 0$ which happens when $f_1$ reaches its maximum value.

\end{proof}


If the maximum value of $f_1$ is less than $\tau_1$, our analysis will be exactly same as the regular fixed step size gradient ascent as the second mode of $\vu$ will not be possible. Otherwise, we can show that $\exists k^* \text{ s.t. } f_1(\vx^{k^*}) \geq \tau_1$. In fact, we can give an upper bound to $k^*$ utilizing known results for fixed step size gradient ascent.
\begin{theorem}
If $f_1$ is convex, differentiable, and $L$-smooth, running gradient ascent algorithm starting from $\vx^{(0)}$ for $i$ steps with step size $\alpha \leq 1/L$ will yield $\vx^{(l)}$ such that:
\begin{equation}
     f_1(\vx^*) - f_1(\vx^{(l)}) \leq \frac{\norm{\vx^{(0)}-\vx^*}^2}{2\alpha i}
\end{equation}
where $\vx^*$ is the optimal solution.
\end{theorem}

Then, we can conclude that 
\begin{align*}
    f_1(\vx^{(l)}) \geq \tau_1  &\iff  f_1(\vx^*) - f_1(\vx^{(l)}) \leq  f_1(\vx^*) - \tau_1 \\
    &\impliedby \frac{\norm{\vx^{(0)}-\vx^*}^2}{2\alpha i} \leq f_1(\vx^*) - \tau_1 \\
    &\impliedby \frac{\norm{\vx^{(0)}-\vx^*}^2}{2\alpha (f_1(\vx^*) - \tau_1)} \leq i 
\end{align*}

Now, we can continue with the more interesting part of our analysis.

If $f_1(\vx^*) \geq \tau_1$, we have shown that our algorithm eventually will switch to the second mode, assigning  $\vu = \projc{\nabla f_2(\vx)}{\nabla f_1(\vx)} = \frac{\cos{\Delta}}{\sin{\phi}} \sin{(\Delta+\phi)} \left(\!\nabla f_2(\vx) +  \nabla f_1(\vx) \frac{\norm{\nabla f_2(\vx)}}{\norm{\nabla f_1(\vx)}} \paren{\sin{\phi}\tan{\Delta} - \cos{\phi}}\! \right)$ where $\Phi$ is the angle between $\nabla f_1(\vx)$ and $\nabla f_2(\vx)$. We will first show that it will not go back to the second mode. In fact, we will prove a stronger claim that $f_1(\vx^{(l+1)}) \geq f_1(\vx^{(l)}) \forall i\geq k^*$.

\begin{equation}\label{eq:f1_ineq}
 f_1(\vx^+)  \geq  f_1(\vx) + \nabla f_1(\vx)^T (\alpha \vu) - \frac{L}{2} \norm{\alpha \vu}^2
\end{equation}

We will compute each term separately. Firstly, we will assume $\nabla f_2 \notin C_{\nabla f_1}^{\Delta}$; hence, requiring actual projection.

\begin{align*}
\;\nabla f_1(\vx)^T &(\alpha \vu) \\
&=\nabla f_1(\vx)^T (\alpha  \frac{\cos{\Delta}}{\sin{\phi}} \sin{(\Delta+\phi)} \left(\!\nabla f_2(\vx) +  \nabla f_1(\vx) \frac{\norm{\nabla f_2(\vx)}}{\norm{\nabla f_1(\vx)}} \paren{\sin{\phi}\tan{\Delta} - \cos{\phi}}\! \right)) \\
&= \alpha \frac{\cos{\Delta}}{\sin{\phi}} \sin{(\Delta+\phi)} \left( \nabla f_1(\vx)^T\nabla f_2(\vx) + \norm{\nabla f_1(\vx)}\norm{\nabla f_2(\vx)}\paren{\sin{\phi}\tan{\Delta} - \cos{\phi}}\right)\\
&= \alpha \frac{\cos{\Delta}}{\sin{\phi}} \sin{(\Delta+\phi)} \left( \norm{\nabla f_1(\vx)}\norm{\nabla f_2(\vx)} \paren{\cos{\phi}+\sin{\phi}\tan{\Delta} - \cos{\phi}}\right) \\
&= \alpha \frac{\cos{\Delta}}{\sin{\phi}} \sin{(\Delta+\phi)} \left( \norm{\nabla f_1(\vx)}\norm{\nabla f_2(\vx)} \paren{\sin{\phi}\tan{\Delta}}\right) \\ 
&= \alpha \cos{\Delta}\tan{\Delta} \sin{(\Delta+\phi)} \left( \norm{\nabla f_1(\vx)}\norm{\nabla f_2(\vx)}\right) \\ 
&= \alpha \sin{\Delta}\sin{(\Delta+\phi)} \left( \norm{\nabla f_1(\vx)}\norm{\nabla f_2(\vx)}\right) \\ 
\end{align*}

\begin{align*}
\frac{L}{2} \norm{\alpha \vu}^2
=&\frac{L}{2} \norm{\alpha \frac{\cos{\Delta}}{\sin{\phi}} \sin{(\Delta+\phi)} \left(\!\nabla f_2(\vx) +  \nabla f_1(\vx) \frac{\norm{\nabla f_2(\vx)}}{\norm{\nabla f_1(\vx)}} \paren{\sin{\phi}\tan{\Delta} - \cos{\phi}}\! \right)}^2 \\
=& \frac{L}{2} \alpha^2 \frac{\cos^2{\Delta}}{\sin^2{\phi}} \sin^2{(\Delta+\phi)} \norm{\left(\!\nabla f_2(\vx) +  \nabla f_1(\vx) \frac{\norm{\nabla f_2(\vx)}}{\norm{\nabla f_1(\vx)}} \paren{\sin{\phi}\tan{\Delta} - \cos{\phi}}\! \right)}^2 \\
=& \frac{L}{2} \alpha^2 \frac{\cos^2{\Delta}}{\sin^2{\phi}} \sin^2{(\Delta+\phi)} (\norm{\nabla f_2(\vx)}^2 + \norm{\nabla f_1(\vx) \frac{\norm{\nabla f_2(\vx)}}{\norm{\nabla f_1(\vx)}} \paren{\sin{\phi}\tan{\Delta} - \cos{\phi}}}^2 + \\ 
& \qquad\qquad\qquad 2 \norm{\nabla f_2(\vx)} \norm{\nabla f_1(\vx) \frac{\norm{\nabla f_2(\vx)}}{\norm{\nabla f_1(\vx)}} \paren{\sin{\phi}\tan{\Delta} - \cos{\phi}}} \cos{\phi} )
\\
= & \frac{L}{2} \alpha^2 \frac{\cos^2{\Delta}}{\sin^2{\phi}} \sin^2{(\Delta+\phi)} \left(\norm{\nabla f_2(\vx)}^2 + \norm{\nabla f_2(\vx)}^2 \paren{\sin{\phi}\tan{\Delta} - \cos{\phi}}^2 + 2 \norm{\nabla f_2(\vx)}^2 \paren{\sin{\phi}\tan{\Delta} - \cos{\phi}} \cos{\phi} \right) \\
= & \frac{L}{2} \alpha^2 \frac{\cos^2{\Delta}}{\sin^2{\phi}} \sin^2{(\Delta+\phi)} \norm{\nabla f_2(\vx)}^2 \left(1 + \paren{\sin{\phi}\tan{\Delta} - \cos{\phi}}^2 + 2 \paren{\sin{\phi}\tan{\Delta} - \cos{\phi}} \cos{\phi} \right)\\
= & \frac{L}{2} \alpha^2 \frac{\cos^2{\Delta}}{\sin^2{\phi}} \sin^2{(\Delta+\phi)} \norm{\nabla f_2(\vx)}^2 \left(\sin^2{\phi} + \cos^2{\phi} + \paren{\sin{\phi}\tan{\Delta} - \cos{\phi}}^2 + 2 \paren{\sin{\phi}\tan{\Delta} - \cos{\phi}} \cos{\phi} \right)\\
= & \frac{L}{2} \alpha^2 \frac{\cos^2{\Delta}}{\sin^2{\phi}} \sin^2{(\Delta+\phi)} \norm{\nabla f_2(\vx)}^2 \left(\sin^2{\phi} + \paren{\cos{\phi} + \paren{\sin{\phi}\tan{\Delta} - \cos{\phi}}}^2 \right)\\
= & \frac{L}{2} \alpha^2 \frac{\cos^2{\Delta}}{\sin^2{\phi}} \sin^2{(\Delta+\phi)} \norm{\nabla f_2(\vx)}^2 \left(\sin^2{\phi} + \sin^2{\phi}\tan^2{\Delta} \right)\\
= & \frac{L}{2} \alpha^2 \frac{\cos^2{\Delta}}{\sin^2{\phi}} \sin^2{(\Delta+\phi)} \norm{\nabla f_2(\vx)}^2 \sin^2{\phi} \left(1 + \tan^2{\Delta} \right)\\
= & \frac{L}{2} \alpha^2 \frac{\cos^2{\Delta}}{\sin^2{\phi}} \sin^2{(\Delta+\phi)} \norm{\nabla f_2(\vx)}^2 \sin^2{\phi} \cos^2{\Delta}\\
= & \frac{L}{2} \alpha^2 \sin^2{(\Delta+\phi)} \norm{\nabla f_2(\vx)}^2\\
\end{align*}

Plugging these in Eq.~\ref{eq:f1_ineq}, we obtain:

\begin{align*}
 f_1(\vx^+)  &\geq  f_1(\vx) + \alpha \sin{\Delta}\sin{(\Delta+\phi)} \left( \norm{\nabla f_1(\vx)}\norm{\nabla f_2(\vx)}\right) - \frac{L}{2} \alpha^2 \sin^2{(\Delta+\phi)} \norm{\nabla f_2(\vx)}^2 \\
 &=  f_1(\vx) + \alpha \sin{(\Delta+\phi)} \norm{\nabla f_2(\vx)} \left( \sin{\Delta} \norm{\nabla f_1(\vx)} - \frac{L}{2} \alpha \sin{(\Delta+\phi)} \norm{\nabla f_2(\vx)}\right) \\ \nonumber
\end{align*}

Hence, when $\sin{\Delta} \norm{\nabla f_1(\vx)} \geq \frac{L}{2} \alpha \sin{(\Delta+\phi)} \norm{\nabla f_2(\vx)} \iff \alpha \leq \frac{2\sin{\Delta} \norm{\nabla f_1(\vx)}}{L\sin{(\Delta+\phi)} \norm{\nabla f_2(\vx)}}$, we will get $f_1(\vx^{(l+1)}) \geq f_1(\vx^{(l)})$. Then, the algorithm does not go back to first mode after switching to the second.

If $\nabla f_2 \in C_{\nabla f_1}^{\Delta}$, $\vu = \nabla f_2$. Then,

\begin{align*}
 f_1(\vx^+)  &\geq  f_1(\vx) + \alpha \cos{\phi}\norm{\nabla f_1(\vx)}\norm{\nabla f_2(\vx)} - \frac{L}{2} \alpha^2 \norm{\nabla f_2(\vx)}^2 \\
&\geq  f_1(\vx) + \alpha \cos{(\frac{\pi}{2}-\Delta)}\norm{\nabla f_1(\vx)}\norm{\nabla f_2(\vx)} - \frac{L}{2} \alpha^2 \norm{\nabla f_2(\vx)}^2 && (\nabla f_2 \in C_{\nabla f_1}^{\Delta} \implies \phi \leq \frac{\pi}{2} - \Delta) \\
&=  f_1(\vx) + \alpha \sin{\Delta}\norm{\nabla f_1(\vx)}\norm{\nabla f_2(\vx)} - \frac{L}{2} \alpha^2 \norm{\nabla f_2(\vx)}^2 \\
&=  f_1(\vx) + \alpha \norm{\nabla f_2(\vx)} \paren{\sin{\Delta}\norm{\nabla f_1(\vx)} - \frac{L}{2} \alpha \norm{\nabla f_2(\vx)}} \\
\end{align*}

Hence, choosing $\alpha$ as described above will guarantee  $f_1(\vx^{(l+1)}) \geq f_1(\vx^{(l)})$.

Now, we will investigate the changes in the value of $f_2$ in the second mode. Writing Lemma~\ref{lemma:concave_ub} for $f_2$ with $\vv{y} = \vx^+$, we obtain Eq.~\ref{eq:f2_ineq}.

\begin{equation}\label{eq:f2_ineq}
f_2(\vx^+) \geq f_2(\vx) + \nabla f_2(\vx)^T (\alpha \vv{u}) - \frac{L}{2}\norm{\alpha\vv{u}}^2
\end{equation}

This time we will assume $\nabla f_2 \notin C_{\nabla f_1}^{\Delta}$ and the other case in analogous as we have shown above. We have already computed $\frac{L}{2}\norm{\alpha\vv{u}}^2$. So, we only need $\nabla f_2(x)^T \vv{u}$.

\begin{align*}
&\nabla f_2(x)^T \vv{u} \\
&= \frac{\cos{\Delta}}{\sin{\phi}} \sin{(\Delta+\phi)} \nabla f_2(x)^T (\!\nabla f_2(x) + \\
& \qquad\quad \nabla f_1(x) \frac{\norm{\nabla f_2(x)}}{\norm{\nabla f_1(x)}}
 \paren{\sin{\phi}\tan{\Delta} - \cos{\phi}}\!)\\
&= \frac{\cos{\Delta}}{\sin{\phi}} \sin{(\Delta+\phi)}(\nabla f_2(x)^T \nabla f_2(x) + \\
& \qquad\quad \nabla f_2(x)^T \nabla f_1(x) \frac{\norm{\nabla f_2(x)}}{\norm{\nabla f_1(x)}} \paren{\sin{\phi}\tan{\Delta} - \cos{\phi}}\!)) \\
&= \frac{\cos{\Delta}}{\sin{\phi}} \sin{(\Delta+\phi)}(\norm{\nabla f_2(x)}^2 + \\
& \qquad\quad \norm{\nabla f_2(x)} \norm{\nabla f_1(x)} \cos{\phi} \frac{\norm{\nabla f_2(x)}}{\norm{\nabla f_1(x)}} \paren{\sin{\phi}\tan{\Delta} - \cos{\phi}}\!)) \\
&= \frac{\cos{\Delta}}{\sin{\phi}} \sin{(\Delta+\phi)}(\norm{\nabla f_2(x)}^2 + \\
& \qquad\quad \norm{\nabla f_2(x)}^2 \cos{\phi} \paren{\sin{\phi}\tan{\Delta} - \cos{\phi}}\!)) \\
&= \frac{\cos{\Delta}}{\sin{\phi}} \sin{(\Delta+\phi)} \norm{\nabla f_2(x)}^2 (1 + 
 \cos{\phi} \paren{\sin{\phi}\tan{\Delta} - \cos{\phi}}\!)) \\
&= \frac{\cos{\Delta}}{\sin{\phi}} \sin{(\Delta+\phi)} \norm{\nabla f_2(x)}^2 (1 -  
\cos{\phi}^2 + \cos{\phi}\sin{\phi}\tan{\Delta} ) \\
&= \frac{\cos{\Delta}}{\sin{\phi}} \sin{(\Delta+\phi)} \norm{\nabla f_2(x)}^2 ( 
\sin{\phi}^2 + \cos{\phi}\sin{\phi}\tan{\Delta} ) \\
&= \cos{\Delta} \sin{(\Delta+\phi)} \norm{\nabla f_2(x)}^2 ( 
\sin{\phi} + \cos{\phi}\tan{\Delta} ) \\
&= \sin{(\Delta+\phi)} \norm{\nabla f_2(x)}^2 ( 
\sin{\phi}\cos{\Delta}  + \cos{\phi}\sin{\Delta} ) \\
&= \sin{(\Delta+\phi)} \norm{\nabla f_2(x)}^2 \sin{(\Delta+\phi)} \\
&= \norm{\nabla f_2(x)}^2 \sin^2{(\Delta+\phi)}
\end{align*}

Now, plugging these values in Eq.~\ref{eq:f2_ineq}, we obtain:
\begin{align*}
    f_2(x^+) &\geq f_2(x) + \alpha\norm{\nabla f_2(x)}^2 \sin^2{(\Delta+\phi)} - \frac{L}{2}\alpha^2\sin^2{(\Delta+\phi)} \norm{\nabla f_2(x)}^2\\
    &=f_2(x) + \alpha(1-\frac{L\alpha}{2})\sin^2{(\Delta+\phi)}\norm{\nabla f_2(x)}^2 \\
    &\geq f_2(x) + \alpha(1-\frac{1}{2})\sin^2{(\Delta+\phi)}\norm{\nabla f_2(x)}^2 && \text{Using $\alpha \leq \frac{1}{L}$}\\
     &= f_2(x) + \frac{\alpha}{2}\sin^2{(\Delta+\phi)}\norm{\nabla f_2(x)}^2
\end{align*}

Note that the second term is always positive when $\nabla f_2(x) \neq 0$ and $\sin{\Delta+\phi} \neq 0$. Since $f_2$ was assumed to be a concave function, $\nabla f_2(x) = 0$ if and only if $x$ is the optimum point of $f_2$. Moreover, as we define $\phi \in [0,\pi]$ and $\Delta \in [0,\pi/2]$, $\sin{\Delta+\phi} = 0$ is possible only if $\Delta+\phi=\pi$ or $\Delta=\phi=0$. The second one is not possible as $\nabla f_2 \notin C_{\nabla f_1}^{\Delta}$. The first case means that this point is $\Delta$-Pareto-stationary. 

\subsection{$K$ objective case}

When generalizing the $K$ objectives, the proofs will be analogous. While the derivation of exact expressions is cumbersome, we can show that Proposition~\ref{prop:alg-monotone-increase} still holds.

With $K$ objectives, our algorithm will pick $\vu$ as follows:

\begin{equation}
    \vu = \begin{cases}
        \nabla f_1(\vx) \;\; & f_1(\vx) < \tau_1 \\
        \projcs{\nabla f_2(\vx)}{\nabla f_1(\vx)} \;\; &f_1(\vx) \geq \tau_1, f_2(\vx) < \tau_2 \\
        \projcs{\projcs{\nabla f_3(\vx)}{\nabla f_1(\vx)}}{\nabla f_2(\vx)} \;\; &f_1(\vx) \geq \tau_1, f_2(\vx) \geq \tau_2, f_3(\vx) < \tau_3  \\
        \dots \\
        \dots \\
        \projcs{\projcs{\cdots\projcs{\nabla f_K(\vx)}{\nabla f_1(\vx)}}{\cdots \nabla f_{K-2}(\vx)}}{\nabla f_{K-1}(\vx)} \;\; &f_i(\vx) \geq \tau_i \;\; \forall i \leq K-1\\       
    \end{cases}
\end{equation}

where $\projcs{}{}$ is a shorthand for $\projc{}{}$. Now, assume that we have the currently optimized objective $i^*_l$ as defined in Proposition~\ref{prop:alg-monotone-increase}. We  
will first show that $f_j(\vx^{(l+1)} \geq f_j(\vx^{(l)}$ for all $j < i^*_l$. 

\begin{equation}\label{eq:satisfied_ineq}
 f_j(\vx^+)  \geq  f_j(\vx) + \nabla f_j(\vx)^T (\alpha \vu) - \frac{L}{2} \norm{\alpha \vu}^2
\end{equation}

We know that unless this is a $\Delta$-Pareto-stationary point, $\vu \in C_{\nabla f_j}^{\Delta}$. Then, $\angle(\vu, \nabla f_j) \leq \frac{\pi}{2}-\Delta$. This is important as it allows us to lower bound $\nabla f_j(\vx)^T(\alpha\vu)$:
\begin{align}
    &\angle(\vu, \nabla f_j) \leq \frac{\pi}{2}-\Delta \nonumber\\
    \implies &\cos{(\angle(\vu, \nabla f_j))} \geq \cos{(\frac{\pi}{2}-\Delta)} = \sin{\Delta}  \nonumber\\
    \implies &\nabla f_j(\vx)^T(\alpha\vu) = \alpha \cos{(\angle(\vu, \nabla f_j)} \norm{\nabla f_j(\vx)}\norm{\vu} \geq \alpha \sin{\Delta} \norm{\nabla f_j(\vx)}\norm{\vu}
\end{align}

Moreover, we have already computed $\norm{\projcs{\nabla f_2(\vx)}{\nabla f_1{\vx}}}$ in Section~\ref{sec:conv-prop:two-obj-case} as $\sin{(\Delta+\phi)} \norm{f_2}$. Then, we now that projecting  $\nabla f_{i^*_l}$ consecutively on ${\nabla f_1, \dots, \nabla f_j}$ gives
\begin{equation}
    \norm{\vu} = \norm{\nabla f_{i^*_l}{\vx}} \prod_{i=1}^j \sin{(\Delta + \phi_i)}
\end{equation}
where $\phi^i = \angle(\nabla f_1(\vx), \vu^{(i)})$ and ${\vu^{(i)}}_{i \in \{1,\dots, j\}}$ is the sequence obtained by the successive projections starting with $\vu^(1) = \nabla f_{i^*_l}(\vx)$.

Then, $\norm{\vu} \leq \norm{\nabla f_{i^*_l}(\vx)}$. Combining these results and plugging in Eq.\ref{eq:satisfied_ineq}:
\begin{align*}
     f_j(\vx^+)  &\geq  f_j(\vx) + \nabla f_j(\vx)^T (\alpha \vu) - \frac{L}{2} \norm{\alpha \vu}^2 \\
     &\geq f_j(\vx) + \alpha \sin{\Delta} \norm{\nabla f_j(\vx)}\norm{\vu} - \frac{L}{2}\alpha^2 \norm{\nabla f_{i^*_l}(\vx)}\norm{\vu} \\
    &= f_j(\vx) + \alpha \norm{\vu} \paren{\sin{\Delta}\norm{\nabla f_j(\vx)} - \frac{L\alpha}{2} \norm{\nabla f_{i^*_l}(\vx)}}
\end{align*}

Hence, when $\sin{\Delta}\norm{\nabla f_j(\vx)} \geq \frac{L\alpha}{2} \norm{\nabla f_{i^*_l}(\vx)}$, the nondecrease in satisfied objectives condition is satisfied.

Showing that $f_{i^*_l}$ increases with a small enough $\alpha$ is analogous, so we omit it here.

\section{Using Lexicographic Projection Algorithm in RL}\label{sec:supp:lpa-rl}

In this section, we first give the REINFORCE algorithm we use as the basis for our Lexicographic REINFORCE algorithm for easier comparison. Then, we share further details of our experiments.

\subsection{REINFORCE Algorithm}\label{sec:pg:reinforce}

The pseudocode for REINFORCE algorithm that we will use as the basis for our adaptation (Algorithm~\ref{alg:lexicographic-reinforce}) can be seen in Algorithm~\ref{alg:vanilla-reinforce}.

\begin{algorithm}[H]
\SetAlgoLined
\SetKwFunction{reinforce}{REINFORCE}
\SetKwProg{Pr}{Process}{:}{}
\Pr{\reinforce}{
Initialize policy function $\pi(a|s, \theta)$ with random parameter $\theta$ \\
\For{ep $= 1, N_e$}{
  Generate an episode $S_0, A_0, R_1, \ldots, S_{T-1}, A_{T-1}, R_T$ and save $\ln \pi(A_t|S_t)$ at every step.\\
  $G_{T+1} \gets 0$ \\
  \For{t $= T, 1$}{
      $G_t \gets R_t + \gamma G_{t+1}$
  }
  $L \gets -\sum_{t=0, T-1} \ln \pi(A_t|S_t) G_{t+1}$ \\
  Update $\theta$ by taking an optimizer step for loss $L$ \\
}
\KwRet $\pi(a|s, \theta)$  \\
}
\caption{Vanilla REINFORCE}
\label{alg:vanilla-reinforce}
\end{algorithm}

Note that Algorithm~\ref{alg:vanilla-reinforce} can be used with optimizers other than vanilla gradient descent. In our experiments, we found that Adam is easier to use with the tasks at hand. Similarly, we found that using Adam optimizer is better than vanilla gradient descent for our adaptation too.

\subsection{Experiments}\label{sec:supp:pg:exp-rl}

In this section, we share the details of the experimental setup for the adapted REINFORCE algorithm.

\subsubsection{Policy Function}\label{sec:supp:pg:exp-rl:policy-function} In both experiments, we use a two layer neural network (\cite{lecun2015deep}) for policy function. We represent the state via one-hot encoding (\cite{harris2015digital}), hence the input dimension is the same as the size of state space. For example, 20 for the maze in Figure~\ref{fig:maze-extended}. Then the hidden layer is a fully connected layer with 128 units and they use \emph{ReLU} activation function \cite{agarap2018deep}. We also used a dropout layer (\cite{srivastava2014dropout}) with drop probability $0.6$. Finally, the output layer has 4 units, representing the four valid actions in our benchmark. The outputs of these units are converted to action probabilities by applying a softmax function with temperature $10$ \cite{lecun2015deep}. The temperature hyperparameter allows making the policy less deterministic by making the action probabilities closer to each other. This makes sure that the policy keeps exploring so it does not get stuck in local minima. This is particularly important for our algorithm, considering that the learning of the less important objectives does not start until the important ones are learned.

\subsubsection{Reachability Experiment}\label{sec:supp:pg:exp:reachability}
For Reachability experiment, we use the maze in Figure~\ref{fig:maze-concave-simple}. As we only care about the agent eventually reaching the goal, the agent can completely avoid going on a bad tile. All the policies where it reaches the goal but goes through a bad tile in the process will be dominated by this policy. Hence, we will expect our agent to learn the policy where it eventually reaches the goal and never steps on a bad tile.


\begin{figure}[tb]
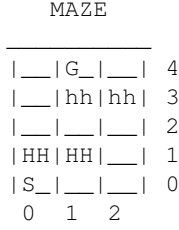

  \begin{CenteredBox}
      \begin{myverbatim}
          MAZE
       __________
       |__|G_|__| 4
       |__|hh|hh| 3
       |__|__|__| 2
       |HH|HH|__| 1
       |S_|__|__| 0
        0  1  2
\end{myverbatim}
\end{CenteredBox}
                        \vspace*{3mm}
                        \caption{The maze to be used in Reachability experiment.} \label{fig:maze-concave-simple}
                                                \vspace*{6mm}
\end{figure}

We run Algorithm~\ref{alg:lexicographic-reinforce} for $N_e = 4000$ episodes and we repeat our experiment with $10$ different random seeds. As the policy we use is stochastic, different seeds give significantly different results. Figure~\ref{fig:reachability-simple-seeds-avg} summarizes the performance of $10$ seeds. The plot shows the ratio of the successful trajectories out of 100 trajectories where successful is defined as satisfying the reachability constraint without stepping on a bad tile. The line shows the mean of 10 different seeds where the shaded region shows the variance in the experiment as two standard deviations around the mean. It can be clearly seen that as the training progresses, the satisfaction frequency increases. Out of the 10 seeds, 4 find policies that have 90\% success over 100 episodes.

\begin{figure}[tb]
	\centering
	\includegraphics[scale=0.5]{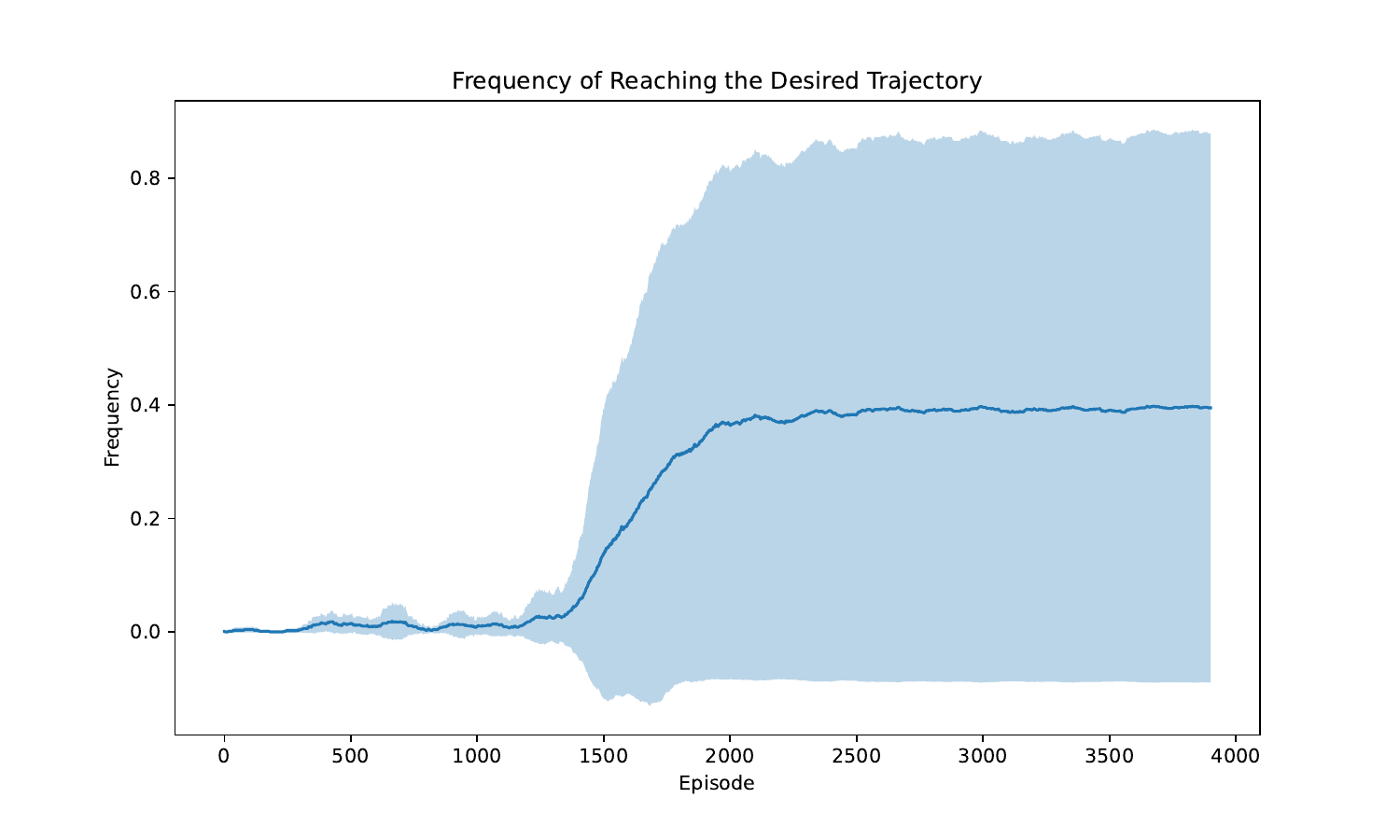}
                                \vspace*{3mm}
                                \caption{Average satisfaction frequency of 10 seeds for the experiment described in Section~\ref{sec:supp:pg:exp:reachability}. The shaded region shows a confidence interval of two standard deviation width around the mean.}
                                                        \vspace*{6mm}
	\label{fig:reachability-simple-seeds-avg}
\end{figure}

We can also take a closer look into how the training progresses for a successful seed. Figure~\ref{fig:reachability-simple-singleseed} shows how the satisfaction frequency for each objective changes throughout the training. It can be seen that the primary objective, reaching the goal eventually starts with a high frequency but drops a little bit while the secondary is being learned. Then, the frequencies for both objectives start to increase together. Intuitively, the initial drop represents when the agent starts to consider "do nothing" policies which reduces the success of the primary objective. But the agent then learns that it can still maintain $0$ penalties without just staying in place.

\end{document}